%% file: main.tex
\documentclass[11pt]{article}
\usepackage{algorithm}
\usepackage{algorithmic}
\usepackage{amsmath}
\usepackage{amssymb}
\usepackage{authblk}
\usepackage{booktabs}
\usepackage[margin=1in]{geometry}
\usepackage{graphicx}
\usepackage{mathtools}
\usepackage{amsthm}
\usepackage[colorlinks=true,citecolor=blue]{hyperref}
\usepackage[section]{placeins}

\newtheorem{proposition}{Proposition}[section]

\title{%
  Decoupled Latent Optimization of Diffusion Models\\
  for Full Waveform Inversion%
}

\author[1]{Chen Min}
\author[1,2]{Zheng Ma\thanks{Corresponding author. E-mail: zhengma@sjtu.edu.cn}}

\affil[1]{School of Mathematical Sciences, Shanghai Jiao Tong University, Shanghai, China}
\affil[2]{CMA-Shanghai, Shanghai Jiao Tong University, Shanghai, China}
\date{}

\begin{document}
\maketitle

\input{content/abstract}


\input{content/introduction}
\input{content/method}
\input{content/experiments}
\input{content/conclusion}

\section*{Acknowledgments}
Zheng Ma is supported by NSFC Grant No. 12531016 and Beijing Institute of Applied
Physics and Computational Mathematics funding HX02023-6. Additionally, we also thank Shanghai Institute for Mathematics
and Interdisciplinary Sciences (SIMIS) for their financial support. This research was funded
by SIMIS under grant number SIMIS-ID-2025-ST. The authors are grateful for the resources
and facilities provided by SIMIS, which were essential for the completion of this work.

\bibliographystyle{unsrt}
\bibliography{main_R2}

\clearpage
\input{content/appendix}

\end{document}

%% file: content/abstract.tex
\begin{abstract}
  Full waveform inversion (FWI) recovers subsurface velocity from seismic
  recordings by solving a severely ill-posed, nonconvex PDE-constrained
  optimization. Classical regularizers stabilize the inversion but fail to
  reproduce realistic geological structures; recent diffusion-prior methods
  improve realism at the cost of a fragile trade-off between data fidelity
  and prior consistency.
  We propose Decoupled Latent Optimization (DLO), which relaxes the
  standard latent-optimization formulation into a quadratic-penalty
  objective over an auxiliary physical variable and a latent variable. The
  data-fidelity gradient acts in physical space, the diffusion sampler
  contributes only through a decoded prior sample, and the standard
  smoothed-velocity initialization of classical FWI is preserved.
  On the OpenFWI benchmark, DLO outperforms classical regularizers and
  existing diffusion-based methods under clean, noisy, and missing-trace
  acquisitions. The prior, trained on $70\times70$ OpenFWI models,
  transfers directly to the Marmousi and Overthrust benchmarks, where DLO
  recovers intricate fault structures and remains robust to initialization
  smoothing and measurement noise.
\end{abstract}

%% file: content/introduction.tex
\section{Introduction}
\label{sec:intro}

Inverse problems governed by partial differential equations (PDEs) play a
central role across science and engineering, encompassing medical imaging,
fluid dynamics, and geophysical exploration~\cite{dashti2013bayesian,tarantola2005inverse}.
A prototypical and particularly demanding instance is seismic inversion, which
is widely used to reconstruct quantitative subsurface models from wavefield
recordings collected at the surface and supports applications ranging from
resource exploration to seismic-hazard assessment and environmental
monitoring~\cite{zhang2021rayleigh,sambridge2002monte}.
Over the past decades, a number of methods have been developed to address
this problem, including velocity analysis on stacked
traces~\cite{berkhout1997pushing}, traveltime tomography and migration-based
methods~\cite{zelt1992seismic,clement2001migration}, and Born-approximation linearized
inversion~\cite{hudson1981use,muhumuza2018seismic}.
Among these methods, full-waveform inversion
(FWI)~\cite{tarantola1984inversion,warner2013anisotropic,jakobsen2015full,virieux2009overview}
is set apart by its ability to recover high-resolution subsurface structure,
which it achieves by casting the inversion as a PDE-constrained optimization
that exploits the full phase and amplitude content of the wavefield to align
the simulated and observed seismograms under the wave
equation~\cite{virieux2009overview,fichtner2010full}.

The optimization problem underlying FWI is widely regarded as a challenging
nonconvex problem, owing to the nonlinearity of its forward modeling and the
ill-posedness of the inverse problem itself.
The associated misfit landscape is strongly nonconvex and exhibits numerous
local minima arising from phase ambiguity and limited subsurface
illumination, leaving the inversion highly sensitive to inaccuracies in the
initial velocity model, measurement noise, and incomplete
acquisitions~\cite{virieux2009overview,pratt1999seismic,yao2019tackling,zhang2020fwi,yazdani2020systems,jiao2021one}.
A particularly characteristic difficulty is cycle skipping, where the
gradient steers the simulated wavefield toward a match that is off by a
whole cycle whenever the predicted and observed traces are misaligned by
more than half a period, and the optimization becomes trapped in an
incorrect basin of attraction, converging to inaccurate saddle
points~\cite{gauthier1986two,pladys2021cycle,metivier2016measuring}.

To mitigate these difficulties and improve robustness against imperfect
observations, a variety of remedies have been explored.
Regularization-based approaches retain the standard quadratic misfit and add
an explicit penalty that constrains the velocity model, with Tikhonov
regularization enforcing smoothness on the recovered
field~\cite{asnaashari2013regularized} and total-variation regularization promoting
piecewise-constant structures that preserve sharp
interfaces~\cite{esser2018total,lin2014acoustic,aghamiry2018hybrid}.
By penalizing non-physical features in the recovered velocity, these
regularizers stabilize the optimization and improve convergence accuracy,
and have accordingly become standard baselines in FWI.
A complementary line of work replaces the pointwise quadratic residual by
misfit functionals that are less sensitive to small time and phase shifts,
including multi-scale frequency continuation~\cite{bunks1995multiscale},
phase-based and envelope-based misfits~\cite{bozdaug2011misfit,chi2014full},
adaptive waveform inversion~\cite{warner2016adaptive}, and geometry-aware
metrics based on optimal transport~\cite{yang2018application,metivier2016measuring}.

Alongside these classical strategies, a substantial body of work has
explored machine-learning approaches to FWI.
Direct supervised mappings learn the inverse operator from synthetic
pairs of seismograms and velocity
models~\cite{wu2019inversionnet,zhang2019velocitygan,zhang2020data,kazei2021mapping}.
These networks achieve much higher computational efficiency than
iterative optimization, and the training data provide implicit
regularization for the ill-posed inverse problem, which yield strong
empirical accuracy on FWI benchmarks.
Physics-aware variants embed the wave equation directly into the
learning process.
Physics-informed neural networks penalize PDE residuals on collocation
points~\cite{karniadakis2021physics,lu2021physics,lu2021deepxde}, and
related schemes reparameterize the velocity field as a neural network
whose weights are updated by the FWI
gradient~\cite{yan2023unsupervised}.
A more recent and particularly fruitful paradigm in PDE solving is
operator learning, in which neural networks approximate
function-to-function mappings between infinite-dimensional
spaces~\cite{lu2021learning,zhu2023fourier}.
Analogous operator architectures have also been investigated for
FWI~\cite{huang2025physics}.
Despite these advances, generalization remains the central limitation of
supervised approaches.
Their accuracy degrades on velocity structures unseen during training
and on data contaminated by measurement noise, and the network typically
must be retrained whenever the acquisition geometry changes.

The rapid development of diffusion models in recent years has raised
generative modeling to a new level and opened a new paradigm for solving
inverse problems.
By learning the score field
$\nabla_{\mathbf{x}}\log p_t(\mathbf{x})$ of a family of
Gaussian-corrupted marginals $\{p_t\}_{t\in[0,T]}$, diffusion models
generate samples by solving the associated reverse-time stochastic
differential equation that transports a standard Gaussian back to the
data distribution~\cite{ho2020denoising,song2020score,song2020denoising}.
From an inverse-problem perspective, a pretrained diffusion model embeds
a data-driven prior into the inversion that provides stronger
regularization than handcrafted alternatives and yields improved
robustness to measurement noise and out-of-distribution observations.
A wide range of diffusion-based algorithms have been witnessed in
image-domain inverse
problems~\cite{kawar2022denoising,mardani2024variational,dou2024diffusion,song2023pseudoinverse,chung2022diffusion,choi2021ilvr},
whereas their application to physical, PDE-governed inverse problems
remains
under-explored~\cite{shu2023physics,shan2026pird,wang2026fundiff}.
For FWI in particular, one line of work inserts FWI physical-constraint
steps into the diffusion sampling trajectory, including DiffusionFWI and
its DiffusionILVR extension~\cite{wang2023prior,taufik2025wavenumber} as well as DPS
variants applied to the wave-equation
operator~\cite{taufik2025diffusion,peng2026robust}.
However, the imposed physical-constraint steps are generally inconsistent
with the Gaussian-noised marginals visited by the sampling flow, and
hyperparameter selection that balances data-prior effectiveness against
physical-constraint enforcement becomes a central difficulty.
A second line of work embeds diffusion-model-based regularization terms
directly into the FWI optimization
objective~\cite{shan2026regularization,xie2025diffusion}.
These works collectively demonstrate that diffusion priors can deliver
meaningful structural information for FWI; the form in which the prior
signal is extracted from the pretrained model, and the way it is injected
into the inversion, remain open design choices that motivate the present
work.

In this paper, we develop a general framework, \emph{Decoupled Latent
Optimization} (DLO), that embeds a pretrained diffusion model into
PDE-governed inverse problems as a regularization mechanism.
DLO introduces an auxiliary optimizable latent variable whose decoded
sample approximates the prior projection of the physical variable,
providing a diffusion-prior regularization signal for the physical
inversion.
The framework inherits the strong generalization to out-of-distribution
structures characteristic of diffusion-prior approaches.

We validate DLO on large-scale FWI benchmarks and compare it against
classical regularizers and existing diffusion-based methods.
On the OpenFWI benchmark~\cite{deng2022openfwi}, DLO consistently outperforms
these baselines in reconstruction accuracy and remains robust under
measurement noise and incomplete acquisitions.
We further apply the framework to the Marmousi~\cite{versteeg1994marmousi} and
Overthrust~\cite{aminzadeh19973} benchmarks that lie clearly outside the
training distribution, and the successful reconstructions on these
models imply that DLO scales to larger domains and adapts well to the
complex geological structures encountered in real field data.

The remainder of the paper is organized as follows.
Section~\ref{sec:background} formulates FWI, reviews diffusion generative
models, and surveys existing diffusion-based approaches for inverse problems.
Section~\ref{sec:method_ours} presents Decoupled Latent Optimization and
analyzes its convergence properties.
Section~\ref{sec:experiments} reports numerical experiments on the OpenFWI,
Marmousi, and Overthrust benchmarks.
Section~\ref{sec:conclusion} concludes with a discussion of limitations and
future directions.
Implementation details, hyperparameters, and additional experimental results
are deferred to the appendices.

%% file: content/method.tex
\section{Background and Preliminaries}
\label{sec:background}

\subsection{Full Waveform Inversion}
\label{sec:method_fwi}

Full waveform inversion (FWI) recovers the subsurface velocity field from seismic wavefield
measurements recorded at the surface~\cite{virieux2009overview,tarantola2005inverse}.
In the acoustic approximation, wave propagation in the subsurface domain $\Omega$ is governed by
\begin{equation}\label{eq:wave}
  \left\{
    \begin{aligned}
      & \frac{1}{v(\mathbf{r})^2}\frac{\partial^2 p(\mathbf{r},t)}{\partial t^2}
      = \nabla^2 p(\mathbf{r},t) + s(\mathbf{r},t,\xi), \\
      & p(\mathbf{r},0) = 0, \\
      & p_t(\mathbf{r},0) = 0.
    \end{aligned}
    \right.
  \end{equation}
  where $\mathbf{r}\in\Omega$ is the spatial coordinate, $t$ is time, $\nabla^2$ is the spatial
  Laplacian, $v(\mathbf{r})$ is the spatially varying subsurface velocity, $p(\mathbf{r},t)$ is the pressure wavefield.
  The source term $s(\mathbf{r}, t, \xi)$ is typically a Ricker wavelet emitted from a point source at the surface, represented as follows:
  \begin{equation}
    s(\mathbf{r},t,\xi) = s_0(t)\,\delta(\mathbf{r} - \xi),
  \end{equation}
  with $s_0(t)=(1 - 2\pi^2 f^2 t^2)e^{-\pi^2 f^2 t^2}$ being the amplitude of the Ricker wavelet with frequency $f$, $\delta$ being the Dirac delta function, $\xi$ being the location of the source.
  Following the OpenFWI framework, we use absorbing boundary layers to simulate wave propagation in an unbounded medium, suppressing artificial reflections from the domain boundaries.

  In a typical surface acquisition the medium is probed by a set of sources
  $\{\xi_i\}_{i=1}^{n_s}$. The wavefield is observed through a
  measurement operator $\mathcal{M}$ that measures the wavefield at the receiver
  points $\{\mathbf{r}_j\}_{j=1}^{n_r}$ placed along the surface at a shallow depth,
  $\mathcal{M}\,p(\mathbf{r}, \cdot,t;v,\xi_i)=\{p(\mathbf{r}_j,t;v,\xi_i)\}_{j=1}^{n_r}$.
  For a source $\xi_i$, solving the wave equation~\eqref{eq:wave} with velocity
  $v$ yields the wavefield $p(\mathbf{r},t;v,\xi_i)$, whose measurement is the
  simulated record. Collecting these records over all sources defines the
  forward operator
  \begin{equation}
    \label{eq:sampling}
    \mathcal{F}(v) = \bigl\{\,p(\mathbf{r}_j,t;v,\xi_i)\,\bigr\}_{j=1,\dots,n_r,\;i=1,\dots,n_s},
    \qquad t\in[0,T],
  \end{equation}
  which maps a velocity model to the corresponding wavefield sampled at every
  source--receiver pair.
  Generally, the observed data $\mathbf{d}_\mathrm{obs}$ are the wavefield generated by the
  true velocity model $v^\ast$, measured at the same receivers and corrupted by
  additive measurement noise,
  \begin{equation}
    \label{eq:dobs}
    \mathbf{d}_\mathrm{obs}
    = \bigl\{\,p(\mathbf{r}_j,t;v^\ast,\xi_i) + \eta_i(\mathbf{r}_j,t)\,\bigr\}_{j=1,\dots,n_r,\;i=1,\dots,n_s},
    \qquad \eta_i\sim\mathcal{N}(0,\sigma^2 I),
  \end{equation}
  where $\eta_i$ denotes the measurement noise at each receiver.
  Fig.~\ref{fig:fwi_overview} illustrates this forward-modeling and
  inversion setup.

  \begin{figure}[!htbp]
    \centering
    \includegraphics[width=0.98\linewidth]{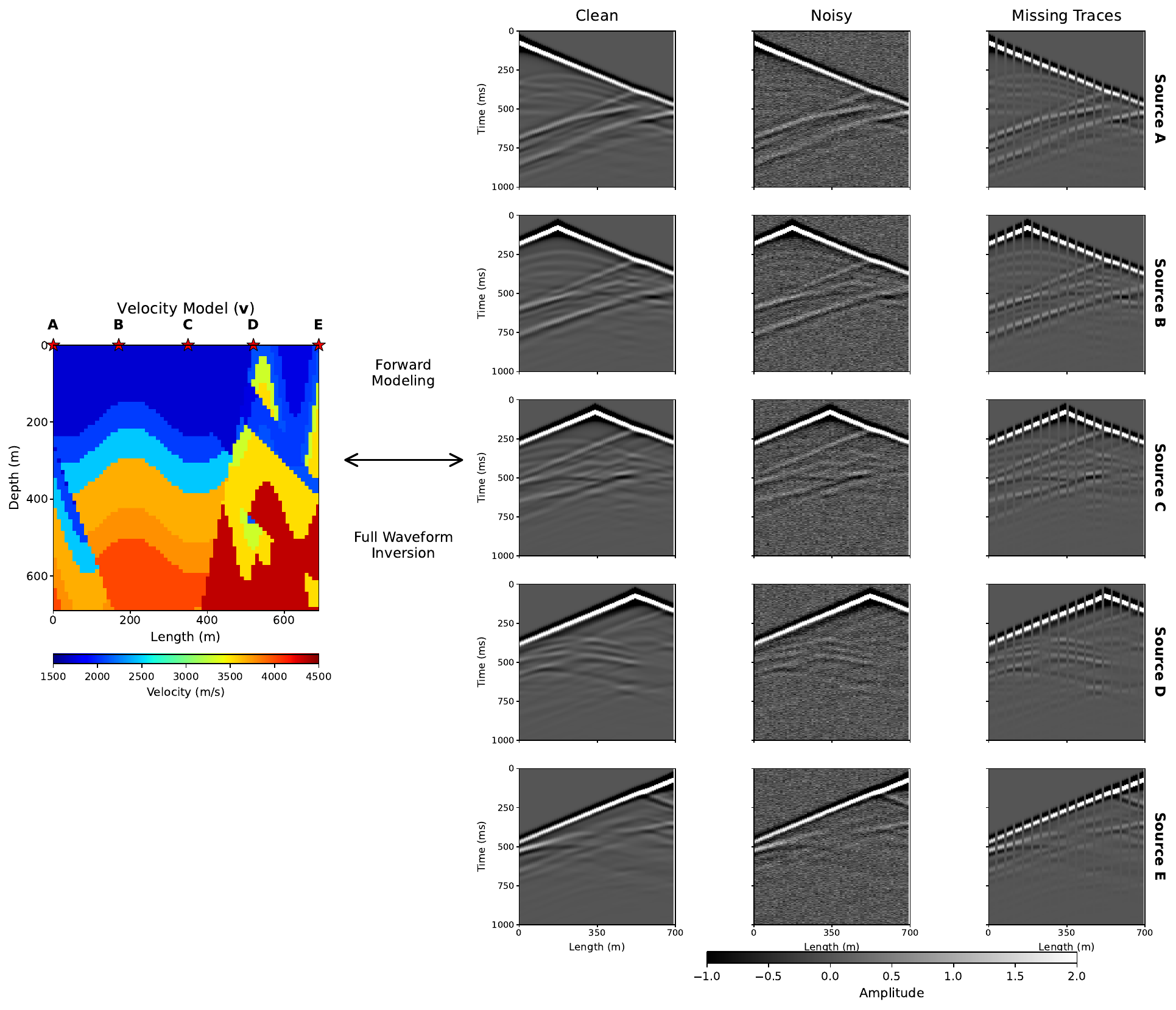}
    \caption{\textbf{Forward modeling and full waveform inversion.}
      A velocity model (left) is excited by five surface sources A--E; the wave
      equation~\eqref{eq:wave} propagates each source through the medium, and
      the resulting wavefield is recorded through time at surface receivers
      (right; rows are sources A--E, columns show the shot gathers under clean,
      noisy, and missing-trace acquisition). FWI inverts this mapping to recover
    $v$ from the observed wavefield.}
    \label{fig:fwi_overview}
  \end{figure}

  Traditional FWI recovers the subsurface velocity $v$ by solving an optimization problem initialized from a rough velocity model (typically a smoothed approximation of the true geology).
  \begin{equation}
    \label{eq:fwi}
    \min_{v}\;
    \bigl\|\mathcal{F}(v) - \mathbf{d}_\mathrm{obs}\bigr\|^2
    + \lambda\,R(v),
  \end{equation}
  where the first term measures the misfit between the observed data $\mathbf{d}_\mathrm{obs}$ and
  the simulated data $\mathcal{F}(v)$, and the second term $R(v)$ is a regularization term weighted by
  $\lambda$ that constrains the solution to be physically plausible and improves the convergence of
  the inversion.
  The gradient of the data misfit is usually computed via the adjoint-state
  method~\cite{tarantola2005inverse}.

  Due to the nonlinearity of the wave equation and the limited coverage of surface observations, FWI
  is a severely ill-posed problem whose optimization landscape contains numerous local minima
  (e.g., cycle-skipping~\cite{pladys2021cycle}).
  Although regularization penalizes non-physical solutions and improves convergence, performance
  remains sensitive to the choice of initial model, the strength and form of the regularization term,
  and measurement noise in the observations, making accurate recovery challenging in practice.

  While classical regularizers such as Tikhonov~\cite{asnaashari2013regularized} and Total Variation
  (TV)~\cite{esser2018total,lin2014acoustic,aghamiry2018hybrid} improve inversion stability, they often fail to produce
  geologically realistic models: Tikhonov tends to over-smooth fine-scale features, while TV
  introduces staircase artifacts at smooth transitions.
  In recent years, deep-learning-based methods have emerged as a promising direction, leveraging
  large-scale data to introduce data-driven regularization that captures complex geological structures
  beyond the reach of handcrafted priors.

  \subsection{Diffusion Generative Models}
  \label{sec:diffusion_prelim}

  A diffusion model defines a continuous family of Gaussian-corrupted marginals
  $\{p_t\}_{t\in[0,T]}$ by progressively adding noise to clean samples
  $x_0\sim p_\mathrm{data}$.
  Under the variance-preserving (VP) corruption~\cite{ho2020denoising,song2020score},
  \begin{equation}
    \label{eq:vp_marginal}
    x_t = \sqrt{\bar\alpha_t}\,x_0 + \sqrt{1-\bar\alpha_t}\,\epsilon,
    \qquad \epsilon\sim\mathcal{N}(0,I),
  \end{equation}
  where $\bar\alpha_t\in(0,1]$ decreases monotonically from $\bar\alpha_0\!=\!1$ to
  $\bar\alpha_T\!\approx\!0$.

  Eq.~\eqref{eq:vp_marginal} yields a sample $x_t\sim p_t$ at each noise
  level, and the family $\{p_t\}_{t\in[0,T]}$ can be regarded as the marginal
  distributions of a continuous-time forward diffusion process described by the
  It\^o stochastic differential equation~\cite{song2020score}
  \begin{equation}
    \label{eq:fwd_sde}
    \mathrm{d}x = f(x,t)\,\mathrm{d}t + g(t)\,\mathrm{d}w,
  \end{equation}
  where $f(\cdot,t)$ is the drift coefficient, $g(t)$ the diffusion coefficient,
  and $w$ a standard Wiener process.
  The VP corruption~\eqref{eq:vp_marginal}
  corresponds to the choice $f(x,t)=-\dfrac{1}{2}\beta(t)\,x$ and
  $g(t)=\sqrt{\beta(t)}$.
  Here $\beta(t)$ is the continuous-time noise-rate schedule, related to the
  corruption coefficient $\bar\alpha_t$ in Eq.~\eqref{eq:vp_marginal} by
  $\beta(t)=-\dfrac{\dot{\bar\alpha}_t}{\bar\alpha_t}$, equivalently
  $\displaystyle\bar\alpha_t=\exp\left(-\int_0^t\beta(s)\,\mathrm{d}s\right)$, so that the two
  schedules encode the same noising process.
  As $t\to T$ the signal-to-noise ratio of the marginal
  $\dfrac{\bar\alpha_t}{1-\bar\alpha_t}$ decreases monotonically to zero, so that
  $p_T\approx\mathcal{N}(0,I)$ is a tractable Gaussian prior.

  A diffusion model generates samples through the reverse form of this SDE,
  gradually denoising a Gaussian sample $x_T\sim\mathcal{N}(0,I)$ back to the
  data distribution by integrating the reverse-time SDE~\cite{ANDERSON1982313}
  \begin{equation}
    \label{eq:rev_sde}
    \mathrm{d}x = \bigl[\,f(x,t) - g(t)^2\,\nabla_{x}\log p_t(x)\,\bigr]\mathrm{d}t
    + g(t)\,\mathrm{d}\bar w,
  \end{equation}
  where $\bar w$ is a standard Wiener process running backwards in time and
  $\mathrm{d}t$ is an infinitesimal negative timestep; the only unknown is the
  score $\nabla_{x}\log p_t(x)$ of each marginal.
  We fit this score by score matching, realized by training a neural network
  $D_\theta(x_t,t)$ to predict the clean sample $x_0$ from the noised observation
  $x_t$ through the VP denoising objective
  \begin{equation}
    \label{eq:denoise_loss}
    \min_{\theta}\;
    \mathbb{E}_{\,x_0\sim p_\mathrm{data},\;t\sim\mathcal{U}(0,T),\;\epsilon\sim\mathcal{N}(0,I)}
    \bigl[\,\omega(t)\,\|D_\theta(x_t,t) - x_0\|^2\,\bigr],
    \quad x_t=\sqrt{\bar\alpha_t}\,x_0+\sqrt{1-\bar\alpha_t}\,\epsilon,
  \end{equation}
  with $\omega(t)>0$ a weighting function; its optimum recovers
  $\mathbb{E}[x_0\!\mid\!x_t]$ and hence the score $\nabla_{x_t}\log p_t(x_t)$
  through Tweedie's formula.

  The sampling process is realized either by the reverse-time SDE~\eqref{eq:rev_sde}
  itself or by its equivalent probability-flow ODE derived from the
  Fokker--Planck equation~\cite{song2020score}, thereby transporting a standard
  Gaussian sample to the clean data distribution.
  Here we consider the standard DDIM sampling process~\cite{song2020denoising}, a
  discretization of the probability-flow ODE.
  We introduce the notation $R_\theta\!:\!z\!\mapsto\!x_0$ to represent this
  deterministic mapping from the initial Gaussian latent
  $x_T=z\sim\mathcal{N}(0,I)$ to the clean sample $x_0$, defined by iterating
  \begin{equation}
    \label{eq:ddim_update_main}
    x_{t-1} = \sqrt{\bar\alpha_{t-1}}\,D_\theta(x_t,t)
    + \sqrt{1-\bar\alpha_{t-1}}\,
    \frac{x_t-\sqrt{\bar\alpha_t}\,D_\theta(x_t,t)}{\sqrt{1-\bar\alpha_t}}.
  \end{equation}
  We treat $R_\theta$ as a learned generator that transports the standard Gaussian
  to (an approximation of) the data distribution $p_\mathrm{data}$; every output
  $R_\theta(z)$ is by construction a sample of the learned prior.
  For the specific theory and detailed derivations of diffusion models, we refer to
  Appendix~\ref{sec:bg_diffusion}.

  \subsection{Diffusion Priors for Inverse Problems}

  Diffusion models provide an efficient and high-quality mechanism for sampling
  from the data prior.
  Embedding this pretrained prior into the solution of inverse problems, however,
  requires carefully designed algorithms that specify how the prior signal is
  extracted from the pretrained model and how it interacts with the physics-based
  data fidelity.
  Existing approaches inject the diffusion prior into inversion in several
  broadly different ways.

  \paragraph{Physical constraint embedded in the reverse sampling process.}
  A classical realization in image-domain inverse problems is
  DPS~\cite{chung2022diffusion} / $\Pi$GDM~\cite{song2023pseudoinverse},
  which approximate the conditional
  score of the noisy marginal by Bayes' rule,
  \begin{equation}
    \label{eq:cond_score}
    \nabla_{x_t}\log p_t\bigl(x_t\!\mid\!\mathbf{d}_\mathrm{obs}\bigr)
    = \nabla_{x_t}\log p_t(x_t)
    + \nabla_{x_t}\log p_t \bigl(\mathbf{d}_\mathrm{obs}\!\mid\!x_t\bigr),
  \end{equation}
  where the intractable likelihood term $\nabla_{x_t}\log p_t\bigl(\mathbf{d}_\mathrm{obs}\!\mid\!x_t\bigr)$ is approximated by evaluating the forward operator at the denoised estimate $\mathbb{E}[x_0\!\mid\!x_t]$.
  A separate line of work, represented by DiffusionFWI~\cite{wang2023prior} and
  DiffusionILVR~\cite{taufik2025wavenumber}, bypasses the conditional score:
  instead, an FWI gradient step (or a low-pass replacement) is interleaved
  directly between successive denoising steps of the unconditional sampler.
  To respect both the prior and the observations, these methods must balance
  the physical constraint against the diffusion prior at every sampling step.
  However, this balance is fragile: the likelihood approximation is inaccurate,
  and the physics-based constraint is inherently inconsistent with the
  noise-perturbed marginals $\{p_t\}$ encountered during sampling.
  As a result, the schemes are sensitive to the choice of guidance strategy
  and hyperparameters, particularly for nonlinear PDE-governed inverse
  problems such as FWI.

  \paragraph{Pretrained Denoiser as a Learned Regularizer}
  A second class of methods retains the physical-space FWI objective and
  introduces the diffusion prior as a regularizer evaluated through the
  pretrained denoiser:
  \begin{equation}
    \label{eq:diffusion_reg}
    \min_{v}\;\bigl\|\mathcal{F}(v)-\mathbf{d}_\mathrm{obs}\bigr\|^2
    + \lambda\,\mathcal{L}\bigl(v,\,\tilde v_\theta(v)\bigr),
  \end{equation}
  where $\tilde v_\theta(v)$ is a prior velocity field associated with the
  current iterate $v$, computed from $v$ through one or more evaluations
  of the pretrained denoiser $D_\theta$; it can be read as a
  denoiser-based prior estimate of $v$, and $\mathcal{L}$ penalizes
  deviation of $v$ from $\tilde v_\theta$, supplying a gradient that
  points toward the learned manifold.
  Different instantiations of this framework differ in how
  $\tilde v_\theta(v)$ is constructed from $v$ and $D_\theta$ and in the
  choice of penalty $\mathcal{L}$.
  For example, RED-diff~\cite{mardani2024variational} and its wave-equation counterpart
  RED-DiffEq~\cite{shan2026regularization} corrupt the current iterate
  $v$ to a randomly drawn noise level $t$ with realization $\epsilon$ and
  read off a single denoiser output,
  \begin{equation}
    \label{eq:prior_velocity}
    \tilde v_\theta(v)
    :=\; D_\theta\bigl(\sqrt{\bar\alpha_t}\,v + \sqrt{1-\bar\alpha_t}\,\epsilon,\;t\bigr),
    \qquad \epsilon\sim\mathcal{N}(0,I),
  \end{equation}
  with $\mathcal{L}$ taken from a variational lower bound on the diffusion
  log-likelihood.
  A common characteristic of both approaches is that the optimized
  solution is not drawn from the diffusion prior. As a result, the
  regularizing effect of the learned prior is indirect, and
  reconstruction quality can depend on the choice of regularization
  form, guidance schedule, and hyperparameters.

  \paragraph{Latent Optimization through the Sampler}
  Another line of work, applicable to inverse problems with a differentiable
  forward operator, replaces the unknown field by the latent input to the
  diffusion sampler and minimizes the data misfit through the composed
  map $\mathcal{F}\circ R_\theta$:
  \begin{equation}
    \label{eq:latent_opt_direct}
    \min_{z}\;\bigl\|\mathcal{F}(R_\theta(z))-\mathbf{d}_\mathrm{obs}\bigr\|^2,
  \end{equation}
  in which the sampler $R_\theta$ acts as a preconditioner that
  constrains every iterate $R_\theta(z)$ to lie on the learned prior
  manifold $\mathcal{M} = \bigl\{\,R_\theta(z)\mid z\sim\mathcal{N}(0,I)\,\bigr\}$.
  This paradigm, pioneered by D-Flow~\cite{ben2024d} and
  DMPlug~\cite{wang2024dmplug} for image-domain inverse problems, achieves
  strict prior consistency without requiring a hand-designed regularizer
  or a guidance schedule.

  However, for physical inverse problems governed by partial differential
  equations such as FWI, directly applying the latent-optimization
  formulation of Eq.~\eqref{eq:latent_opt_direct} faces two
  fundamental obstacles.
  First, the forward operator $\mathcal{F}$ involves solving a strongly
  nonlinear PDE; composing it with the sampler $R_\theta$ chains two
  nonlinear maps, and the resulting loss landscape combines the
  cycle-skipping minima of classical FWI with the nonconvexity introduced
  by backpropagation through the PF-ODE. This introduces additional saddle
  points and destabilizes gradients.
  Second, classical FWI relies on a physically informed
  initialization---typically a smoothed version of the true velocity---to
  land in a favorable basin of attraction. The strict latent formulation,
  however, constrains every iterate to the learned prior manifold
  $\mathcal{M}$ and precludes such physical-space initialization: the
  natural choice $z\sim\mathcal{N}(0,I)$ produces only random prior
  samples, yielding optimization that is sensitive to the initial seed.

  These challenges motivate the decoupled formulation we introduce in the
  next section.

  \section{Decoupled Latent Optimization for FWI}
  \label{sec:method_ours}

  To solve highly nonlinear and ill-posed inverse problems such as FWI, and to
  address these two obstacles, we develop \emph{Decoupled Latent Optimization}
  (DLO).

  We begin by observing that the latent-optimization
  problem~\eqref{eq:latent_opt_direct} is equivalent to the
  equality-constrained formulation
  \begin{equation}
    \label{eq:constrained_form}
    \min_{v,\,z}\;
    \bigl\|\mathcal{F}(v) - \mathbf{d}_\mathrm{obs}\bigr\|^2
    \quad\text{subject to}\quad
    v = R_\theta(z),
  \end{equation}
  since the constraint eliminates the auxiliary variable $v$ and recovers
  Eq.~\eqref{eq:latent_opt_direct} exactly.
  We then relax the hard constraint $v = R_\theta(z)$ by introducing a
  quadratic penalty term, yielding the decoupled objective
  \begin{equation}
    \label{eq:dlo}
    \min_{v,\,z}\;
    \bigl\|\mathcal{F}(v) - \mathbf{d}_\mathrm{obs}\bigr\|^2
    + \lambda\,\bigl\|v - R_\theta(z)\bigr\|^2,
  \end{equation}
  where $\lambda>0$ is a fixed penalty parameter.
  Eq.~\eqref{eq:dlo} is the standard quadratic penalty
  method~\cite{bertsekas1999nonlinear} applied to the constrained
  problem~\eqref{eq:constrained_form}; a rigorous analysis of the
  relationship between the two formulations is given in
  Section~\ref{sec:app_penalty}.

  This decoupling avoids gradient backpropagation through the composed
  PDE-solver--sampler chain and permits flexible physical-space
  initialization. The $z$-update searches for the projection of $v$ onto
  the prior manifold $\mathcal{M}$, while the physics-based optimization
  of $v$ is regularized by the prior velocity field $R_\theta(z)$.

  In practice, we initialize $v$ with a Gaussian-smoothed version of the
  true velocity (the standard starting point in classical FWI) and draw
  $z^{(0)}\sim\mathcal{N}(0,I)$. The penalty parameter $\lambda$ is held
  fixed, avoiding the progressive ill-conditioning that arises when
  $\lambda$ is driven to infinity in classical penalty schemes.

  \subsection{Alternating Optimization}

  In practice we minimize Eq.~\eqref{eq:dlo} by alternating gradient descent on
  $v$ and $z$.
  At iteration $k$, the physical variable $v$ is updated by
  \begin{equation}
    \label{eq:v_sub}
    \mathcal{L}_v\bigl(v;\,z^{(k)}\bigr)
    = \bigl\|\mathcal{F}(v) - \mathbf{d}_\mathrm{obs}\bigr\|^2
    + \lambda\,\bigl\|v - R_\theta(z^{(k)})\bigr\|^2,
  \end{equation}
  whose gradient with respect to $v$ is
  \begin{equation}
    \label{eq:v_grad}
    \nabla_{\!v}\mathcal{L}_v
    = \nabla_{\!v}\bigl\|\mathcal{F}(v) - \mathbf{d}_\mathrm{obs}\bigr\|^2
    + 2\lambda\bigl(v - R_\theta(z^{(k)})\bigr),
  \end{equation}
  where the data-misfit gradient is computed efficiently by the
  adjoint-state method~\cite{tarantola2005inverse}.
  The $v$-update is then
  \begin{equation}
    \label{eq:v_update}
    v^{(k+1)} = v^{(k)} - \eta_v\,\nabla_{\!v}\mathcal{L}_v\bigl(v^{(k)};z^{(k)}\bigr).
  \end{equation}
  Given the updated physical variable $v^{(k+1)}$, the latent variable $z$ is
  updated by minimizing the coupling loss
  \begin{equation}
    \label{eq:z_sub}
    \mathcal{L}_z\bigl(z;\,v^{(k+1)}\bigr)
    = \bigl\|R_\theta(z) - v^{(k+1)}\bigr\|^2.
  \end{equation}

  To compute the gradient of this objective, we backpropagate through
  the deterministic DDIM chain.
  Let the sampler be unrolled over a decreasing sequence of timesteps
  $T = t_0 > t_1 > \cdots > t_n = 0$.
  Differentiating the DDIM update~\eqref{eq:ddim_update_main} with respect to
  $x_{t_i}$ gives the per-step Jacobian
  \begin{equation}
    \label{eq:ddim_jacobian_step}
    \frac{\partial x_{t_{i-1}}}{\partial x_{t_i}}
    = \Bigl(\sqrt{\bar\alpha_{t_{i-1}}}
      - \frac{\sqrt{\bar\alpha_{t_i}}\sqrt{1-\bar\alpha_{t_{i-1}}}}
    {\sqrt{1-\bar\alpha_{t_i}}}\Bigr)
    \frac{\partial D_\theta(x_{t_i},t_i)}{\partial x_{t_i}}
    + \frac{\sqrt{1-\bar\alpha_{t_{i-1}}}}{\sqrt{1-\bar\alpha_{t_i}}}\,I.
  \end{equation}
  The Jacobian of the full sampler $R_\theta\!:\!x_T\!\mapsto\!x_0$ is then the
  product of per-step Jacobians,
  \begin{equation}
    \label{eq:ddim_jacobian}
    \frac{\partial R_\theta(z)}{\partial z}
    = \prod_{i=1}^{n}
    \frac{\partial x_{t_{i-1}}}{\partial x_{t_i}},
  \end{equation}
  which is evaluated by automatic differentiation through the unrolled
  computation graph.
  With a small number of steps ($n = 3$), computing this gradient to
  optimize the initial noise $z$ effectively searches for the projection of
  the physical velocity $v$ onto the prior manifold $\mathcal{M}$.
  The $z$-update then reads
  \begin{equation}
    \label{eq:z_update}
    z^{(k+1)} = z^{(k)} - \eta_z\,
    \nabla_{\!z}\bigl\|R_\theta(z^{(k)}) - v^{(k+1)}\bigr\|^2.
  \end{equation}
  The full procedure is summarized in Algorithm~\ref{alg:dlo}.

  \begin{algorithm}[!htbp]
    \caption{Decoupled Latent Optimization (DLO) for FWI}
    \label{alg:dlo}
    \begin{algorithmic}[1]
      \REQUIRE observed data $\mathbf{d}_\mathrm{obs}$; forward operator
      $\mathcal{F}$; pretrained DDIM sampler $R_\theta$ with schedule
      $\{t_i\}_{i=0}^{n}$; regularization weight $\lambda$; learning rates
      $\eta_v,\eta_z$; smoothed initial velocity $v^{(0)}$; number of
      outer iterations $K$.
      \STATE $z^{(0)} \sim \mathcal{N}(\mathbf{0},I)$
      \hfill \COMMENT{latent initialization (sole source of randomness)}
      \FOR{$k = 0,\dots,K-1$}
      \STATE $v^{(k+1)} \gets
      v^{(k)} - \eta_v\,\nabla_v\mathcal{L}_v\bigl(v^{(k)};z^{(k)}\bigr)$
      \hfill \COMMENT{$v$-step on Eq.~\eqref{eq:v_sub}}
      \STATE $z^{(k+1)} \gets
      z^{(k)} - \eta_z\,\nabla_z\mathcal{L}_z\bigl(z^{(k)};v^{(k+1)}\bigr)$
      \hfill \COMMENT{$z$-step on Eq.~\eqref{eq:z_sub}}
      \ENDFOR
      \RETURN $\hat v = v^{(K)}$
    \end{algorithmic}
  \end{algorithm}

  Since DDIM is the deterministic discretization of the
  probability-flow ODE, the reconstruction $\hat v$ is deterministic
  with respect to the initial latent vector $z^{(0)}\sim\mathcal{N}(0,I)$.
  Drawing $z^{(0)}$ independently from $\mathcal{N}(0,I)$ therefore
  induces a distribution over data-consistent velocity models, which is
  the basis of the uncertainty quantification in
  Section~\ref{sec:exp_uq}.

  A schematic overview of the DLO procedure is provided in
  Fig.~\ref{fig:dlo_flowchart}; implementation details, hyperparameter
  settings, and computational cost are given in
  Appendix~\ref{sec:methods_bg}.

  \begin{figure}[!htbp]
    \centering
    \includegraphics[width=0.98\linewidth]{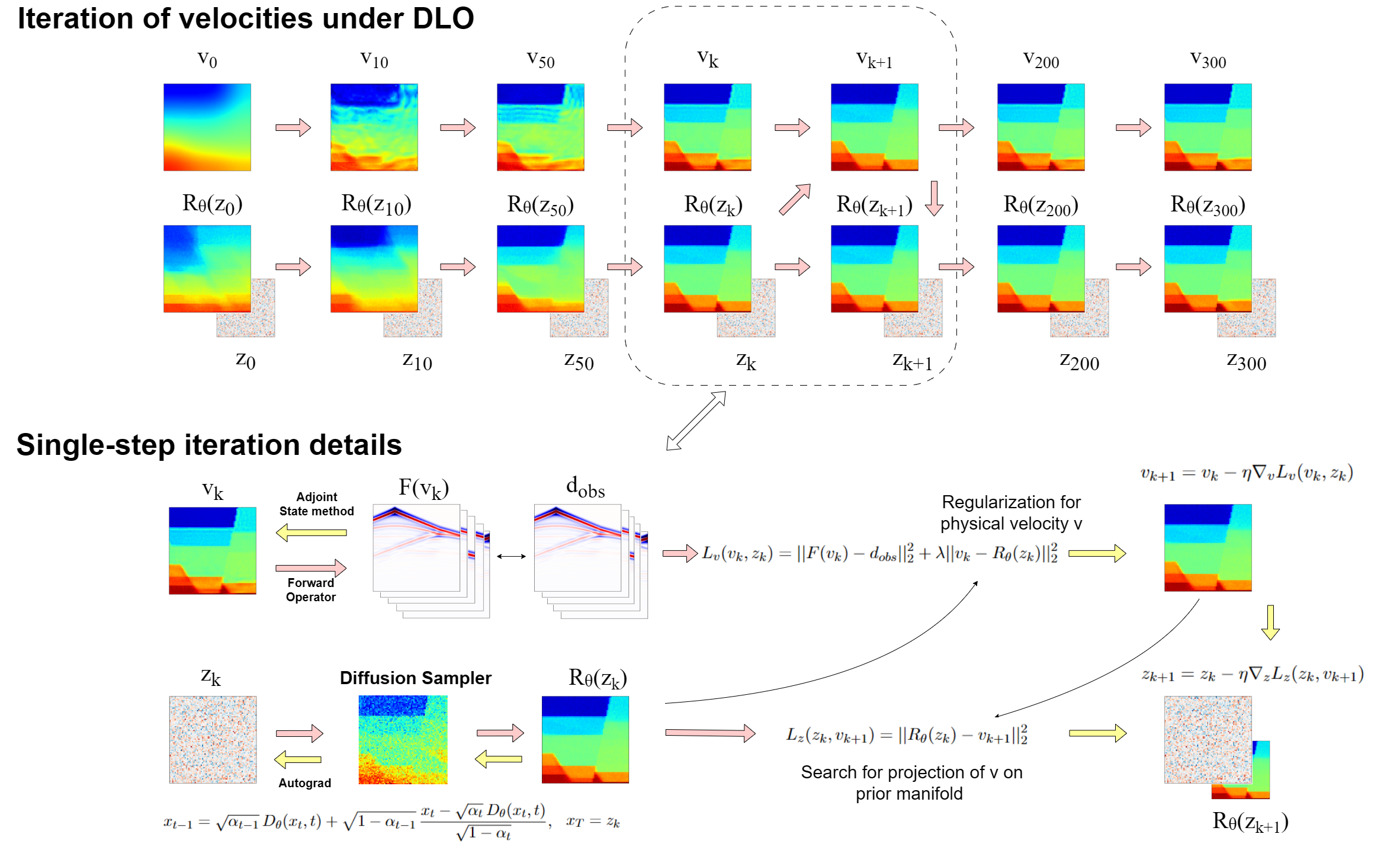}
    \caption{\textbf{An overview of Decoupled Latent Optimization for full waveform
      inversion.}
      \textbf{Top:} the evolution of the physical velocity $v$ and the prior velocity
      $v_\mathrm{gen}=R_\theta(z)$ over the course of the inversion.
      \textbf{Bottom:} a detailed view of a single iteration, showing the alternating
    $v$-update and $z$-update.}
    \label{fig:dlo_flowchart}
  \end{figure}

  \subsection{Convergence of the Penalty Relaxation}
  \label{sec:app_penalty}

  The DLO formulation~\eqref{eq:dlo} replaces the strictly constrained
  latent optimization~\eqref{eq:latent_opt_direct} with the penalty
  objective
  \begin{equation}
    \label{eq:penalty_app}
    \Phi_\lambda(v,z)
    = \bigl\|\mathcal{F}(v) - \mathbf{d}_\mathrm{obs}\bigr\|^2
    + \lambda\,\bigl\|v - R_\theta(z)\bigr\|^2,
  \end{equation}
  which penalizes deviation of the physical velocity $v$ from the decoded
  prior sample $R_\theta(z)$.
  This section provides a rigorous justification for the penalty relaxation
  by establishing that, as $\lambda\to\infty$, the minimizers of
  $\Phi_\lambda$ converge to minimizers of the original constrained problem
  $\min_v\|\mathcal{F}(v)-\mathbf{d}_\mathrm{obs}\|^2$ subject to
  $v=R_\theta(z)$.
  The argument follows the standard quadratic penalty method
  framework~\cite{bertsekas1999nonlinear}.

  We write the data-misfit functional as
  $\mathcal{J}(v)=\|\mathcal{F}(v)-\mathbf{d}_\mathrm{obs}\|^2$, which is
  continuous whenever the forward PDE operator $\mathcal{F}$ is well-posed.
  The DDIM sampler $R_\theta:\mathbb{R}^d\to\mathbb{R}^n$ is a composition
  of affine transforms and pretrained neural-network evaluations and is
  therefore continuous.
  The constrained problem and its optimal value are
  \begin{equation}
    \label{eq:constrained_app}
    \mathcal{J}^*
    = \inf_{v,\,z}\;\{\,\mathcal{J}(v) \mid v = R_\theta(z)\,\}
    = \inf_{z}\,\mathcal{J}(R_\theta(z)).
  \end{equation}
  We assume the feasible set is nonempty and $\mathcal{J}^*>-\infty$.

  \begin{proposition}[Exact minimization]\label{prop:penalty_exact}
    Let $\{\lambda_k\}_{k=0}^\infty$ be a positive sequence with
    $\lambda_k\to\infty$, and for each $k$ let
    \begin{equation}
      \label{eq:exact_min}
      (v^k,\,z^k) \in
      \operatorname*{arg\,min}_{v,\,z}\;
      \mathcal{J}(v) + \lambda_k\,\|v - R_\theta(z)\|^2.
    \end{equation}
    Assume that a global minimizer exists for each $k$.
    Then any limit point of $\{(v^k,z^k)\}$ is a global minimizer of the
    constrained problem~\eqref{eq:constrained_app}, i.e.\ it satisfies
    $\bar v = R_\theta(\bar z)$ and $\mathcal{J}(\bar v) = \mathcal{J}^*$.
  \end{proposition}

  \begin{proof}
    Let $(\bar v,\bar z)$ be a limit point of $\{(v^k,z^k)\}$ and pass to a
    convergent subsequence $(v^k,z^k)\to(\bar v,\bar z)$.
    Since $(v^k,z^k)$ is a global minimizer of $\Phi_{\lambda_k}$,
    \begin{equation}
      \label{eq:pf_glob}
      \mathcal{J}(v^k) + \lambda_k\,\|v^k-R_\theta(z^k)\|^2
      \;\le\; \mathcal{J}(v) + \lambda_k\,\|v-R_\theta(z)\|^2
      \qquad\forall\,(v,z).
    \end{equation}
    Restricting the right-hand side to feasible pairs, for which
    $v=R_\theta(z)$ and the penalty term vanishes, yields
    $\mathcal{J}(v^k)+\lambda_k\|v^k-R_\theta(z^k)\|^2\le\mathcal{J}(v)$
    for every feasible $(v,z)$, and taking the infimum over all such pairs
    gives
    \begin{equation}
      \label{eq:pf_opt}
      \mathcal{J}(v^k) + \lambda_k\,\|v^k-R_\theta(z^k)\|^2
      \;\le\; \mathcal{J}^*.
    \end{equation}
    In particular, the nonnegativity of the penalty term implies
    $\mathcal{J}(v^k)\le\mathcal{J}^*$ for all $k$.

    Rearranging~\eqref{eq:pf_opt} yields the bound
    \begin{equation}
      \label{eq:feas_bound}
      0 \;\le\; \|v^k-R_\theta(z^k)\|^2
      \;\le\; \frac{\mathcal{J}^* - \mathcal{J}(v^k)}{\lambda_k}.
    \end{equation}
    Since $\mathcal{J}^*$ is finite and $\mathcal{J}(v^k)$ is bounded above
    by~\eqref{eq:pf_opt}, the right-hand side tends to zero as
    $\lambda_k\to\infty$; hence
    $\|v^k-R_\theta(z^k)\|\to0$.
    By continuity of $R_\theta$, it follows that
    $\|\bar v-R_\theta(\bar z)\|=0$, i.e.\ $\bar v=R_\theta(\bar z)$;
    thus $(\bar v,\bar z)$ is feasible for the constrained
    problem~\eqref{eq:constrained_app}.

    Finally, taking limits in $\mathcal{J}(v^k)\le\mathcal{J}^*$ and using
    the continuity of $\mathcal{J}$ yields
    $\mathcal{J}(\bar v)\le\mathcal{J}^*$, while feasibility gives
    $\mathcal{J}(\bar v)\ge\mathcal{J}^*$ by definition of the infimum.
    Hence $\mathcal{J}(\bar v)=\mathcal{J}^*$, and $(\bar v,\bar z)$ is a
    global minimizer of~\eqref{eq:constrained_app}.
  \end{proof}

  \begin{proposition}[Inexact minimization]\label{prop:penalty_inexact}
    Assume that $\mathcal{J}$ and $R_\theta$ are continuously
    differentiable.
    Let $\{\lambda_k\}$ be a sequence with $\lambda_k\to\infty$, and let
    $\{(v^k,z^k)\}$ be a sequence of approximate stationary points
    satisfying
    \begin{equation}
      \label{eq:inexact_app}
      \|\nabla\Phi_{\lambda_k}(v^k,z^k)\| \le \varepsilon_k,
      \qquad \varepsilon_k \to 0.
    \end{equation}
    If $(v^k,z^k)\to(\bar v,\bar z)$, then
    \begin{enumerate}
      \item[(i)] $\bar v = R_\theta(\bar z)$;
      \item[(ii)] the vectors
        $\mu^k \coloneqq 2\lambda_k(v^k-R_\theta(z^k))$ converge to
        $\bar\mu\in\mathbb{R}^n$, and
        $\nabla\mathcal{J}(\bar v) + \bar\mu = 0$.
    \end{enumerate}
  \end{proposition}

  \begin{proof}
    The joint gradient of $\Phi_{\lambda_k}$ with respect to $(v,z)$ is
    \begin{equation*}
      \nabla\Phi_{\lambda_k}(v,z)
      =
      \begin{pmatrix}
        \nabla\mathcal{J}(v) + 2\lambda_k(v-R_\theta(z)) \\[4pt]
        -2\lambda_k\,\nabla R_\theta(z)^{\mathsf{T}}(v-R_\theta(z))
      \end{pmatrix}.
    \end{equation*}
    Define $\mu^k \coloneqq 2\lambda_k(v^k-R_\theta(z^k))$.
    Then condition~\eqref{eq:inexact_app} implies
    \begin{align}
      \nabla\mathcal{J}(v^k) + \mu^k &\to 0, \label{eq:gv}\\
      \nabla R_\theta(z^k)^{\mathsf{T}}\mu^k &\to 0. \label{eq:gz}
    \end{align}

    From~\eqref{eq:gv} we have
    $\mu^k = -\nabla\mathcal{J}(v^k) + o(1)$.
    Since $v^k\to\bar v$ and $\nabla\mathcal{J}$ is continuous,
    $\nabla\mathcal{J}(v^k)\to\nabla\mathcal{J}(\bar v)$, and therefore
    \begin{equation}
      \label{eq:mu_conv}
      \mu^k\to\bar\mu \coloneqq -\nabla\mathcal{J}(\bar v).
    \end{equation}
    In particular, $\{\mu^k\}$ is bounded.

    From the definition $\mu^k = 2\lambda_k(v^k-R_\theta(z^k))$,
    \begin{equation}
      \label{eq:feas_inexact}
      \|v^k-R_\theta(z^k)\|
      = \frac{\|\mu^k\|}{2\lambda_k} \;\to\; 0,
    \end{equation}
    so by continuity of $R_\theta$ we obtain $\bar v = R_\theta(\bar z)$,
    establishing (i).

    Finally, since $R_\theta$ is $C^1$, $\nabla R_\theta$ is continuous; together
    with $z^k\to\bar z$ and $\mu^k\to\bar\mu$, we may take the limit
    in~\eqref{eq:gz} to obtain $\nabla R_\theta(\bar z)^{\mathsf{T}}\bar\mu = 0$,
    while~\eqref{eq:gv} gives $\nabla\mathcal{J}(\bar v) + \bar\mu = 0$.
    These are precisely the first-order necessary conditions for the constrained
    problem~\eqref{eq:constrained_app}.
  \end{proof}

  \medskip\noindent\textbf{Remark.}
  Proposition~\ref{prop:penalty_exact} shows that the DLO penalty
  relaxation~\eqref{eq:dlo} is asymptotically exact: driving
  $\lambda\to\infty$ recovers solutions of the original constrained
  latent-optimization problem~\eqref{eq:latent_opt_direct}.
  In the DLO algorithm, $\lambda$ is held fixed rather than taken to
  infinity, so the penalty term acts as a soft constraint whose strength
  is calibrated by a single hyperparameter.
  The constant-$\lambda$ strategy avoids the progressive ill-conditioning
  that plagues classical penalty schemes when the penalty parameter grows
  without bound~\cite{bertsekas1999nonlinear} and in practice a
  moderate value of $\lambda$ is sufficient to steer the physical inversion
  toward the learned velocity manifold while preserving the numerical
  stability of the alternating gradient updates.
  Proposition~\ref{prop:penalty_inexact} further confirms that the
  convergence behaviour is robust to the inexact subproblem solves
  inherent in the finite-step alternating optimization used by DLO.

%% file: content/experiments.tex
\section{Numerical Experiments}
\label{sec:experiments}

We validate DLO on the OpenFWI benchmark~\cite{deng2022openfwi}
(Section~\ref{sec:exp_openfwi}), examine its uncertainty quantification
behaviour induced by the stochasticity of the latent sampler
(Section~\ref{sec:exp_uq}), and finally test its out-of-distribution
generalization on the Marmousi~\cite{versteeg1994marmousi} and
Overthrust~\cite{aminzadeh19973} models (Section~\ref{sec:exp_large}).
Implementation details, including pretraining of the diffusion prior
(Appendix~\ref{sec:methods_diffusion}), forward solver of wave equation (Appendix~\ref{sec:app_solver}), the
hyperparameter settings (Appendix~\ref{sec:app_hyper}), the protocols used to replicate every baseline (Appendix~\ref{sec:app_baselines}), and the
computational cost analysis (Appendix~\ref{sec:app_algo}) are deferred to
the appendix.

\FloatBarrier
\subsection{Evaluation Metrics}
\label{sec:methods_metrics}

We assess reconstruction quality with three standard metrics
computed in the normalized $[-1,1]$ velocity domain: the root mean
square error (RMSE), the mean absolute error (MAE), and the
structural similarity index measure (SSIM).
Let $\{x_i\}_{i=1}^{N}$ and $\{\hat x_i\}_{i=1}^{N}$ denote the
ground-truth and recovered velocity models over the $N$-pixel
spatial domain.
RMSE and MAE measure point-wise discrepancies and are defined by
\begin{equation}
  \label{eq:rmse_mae}
  \mathrm{RMSE} = \sqrt{\frac{1}{N}\textstyle\sum_{i=1}^{N}(x_i-\hat x_i)^{2}},
  \qquad
  \mathrm{MAE}  = \frac{1}{N}\textstyle\sum_{i=1}^{N}\bigl|x_i-\hat x_i\bigr|,
\end{equation}
with RMSE emphasizing larger errors and MAE capturing average
deviations.
SSIM~\cite{wang2004image} compares structural and perceptual similarity
between $x$ and $\hat x$, taking values in $[0,1]$ where higher is
better,
\begin{equation}
  \label{eq:ssim}
  \mathrm{SSIM}(x,\hat x)
  = \frac{(2\mu_x\mu_{\hat x}+C_1)\,(2\sigma_{x\hat x}+C_2)}
  {(\mu_x^{2}+\mu_{\hat x}^{2}+C_1)\,(\sigma_x^{2}+\sigma_{\hat x}^{2}+C_2)},
\end{equation}
where $\mu$ and $\sigma^{2}$ denote local means and variances,
$\sigma_{x\hat x}$ the local covariance, and $C_1,C_2$ small
numerical-stability constants.
Local statistics are evaluated on an $11\!\times\!11$ sliding window
and averaged over the image to give the final score.
RMSE and MAE quantify overall intensity errors, whereas SSIM
captures perceptual and structural consistency, jointly providing
complementary views of reconstruction quality.

\FloatBarrier
\subsection{Validation on OpenFWI Datasets}
\label{sec:exp_openfwi}

OpenFWI~\cite{deng2022openfwi} is a collection of large-scale, multi-structural
benchmark datasets that provide paired velocity models and simulated seismic
recordings across a broad range of synthetic subsurface geologies, and it has
become the standard testbed for diffusion-based FWI.
We select four representative families from the benchmark whose velocity
fields cover complementary geological regimes:
FlatVel-B (FV-B), whose velocity fields consist of horizontally layered
structures with sharp velocity contrasts across flat interfaces;
FlatFault-B (FF-B), whose velocity fields contain horizontally layered
backgrounds disrupted by planar faults that introduce piecewise
discontinuities;
CurveVel-B (CV-B), whose velocity fields exhibit curved, dipping layers with
smooth lateral velocity variations;
and CurveFault-B (CF-B), whose velocity fields combine curved layered
structures with complex curved fault networks.
We compare DLO against the following baselines: standard FWI without
regularization, FWI with Tikhonov regularization~\cite{asnaashari2013regularized}, FWI with total
variation (TV)~\cite{esser2018total}, DiffusionFWI~\cite{wang2023prior}, and
RED-DiffEq~\cite{shan2026regularization}, the state-of-the-art among
diffusion-based FWI methods.

For each family we independently train a variance-preserving diffusion model
on its training set following DDPM~\cite{ho2020denoising}, and reuse the trained model
for every diffusion-based method evaluated on that family.
The network architecture, noise schedule, and training hyperparameters are
detailed in Section~\ref{sec:methods_diffusion}.
Fig.~\ref{fig:datasets_vs_generated} shows ground-truth velocity models
from each family alongside samples generated by the corresponding diffusion
prior, confirming that the learned generators reproduce the geological
characteristics of the four families.

\begin{figure}[!htbp]
  \centering
  \includegraphics[width=0.98\linewidth]{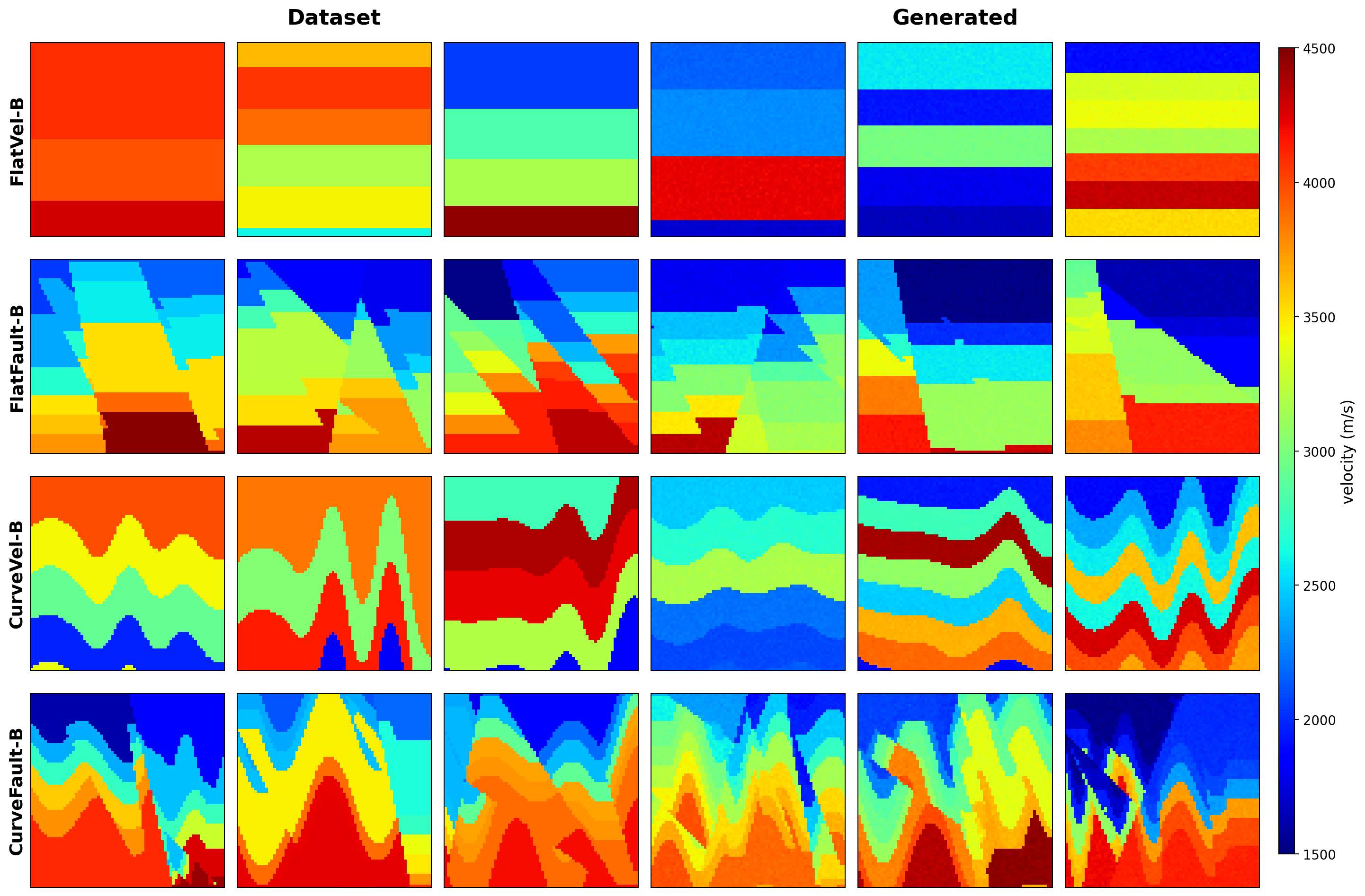}
  \caption{Comparison between velocity models from training set (left) and unconditionally generated samples from each pretrained diffusion model (right)}
  \label{fig:datasets_vs_generated}
\end{figure}

\FloatBarrier
\subsubsection{Clean Seismic Data}
\label{sec:exp_openfwi_clean}

We evaluate every method on $100$ velocity--waveform pairs drawn from the
test set of each family; all reported metrics (defined in
Section~\ref{sec:methods_metrics}) are averages over these $100$
samples within the corresponding subdataset.
All optimization-based methods are initialized from a Gaussian-smoothed
ground-truth velocity with standard deviation $\sigma=10$ and optimized for $300$
iterations for a fair comparison.
For DiffusionFWI~\cite{wang2023prior}, we follow its official repository
settings: the reverse process starts from the intermediate noise level
$t=100$ and performs $10$ FWI optimization steps between consecutive
diffusion steps.
For RED-DiffEq~\cite{shan2026regularization}, we adopt its original
hyperparameter settings and optimize for $300$ iterations.
We provide detailed settings of the numerical scheme of the acoustic
wave-equation forward solver in Appendix~\ref{sec:app_solver}, and the
hyperparameter settings of all these experiments are specifically
discussed in Appendix~\ref{sec:app_hyper}.

Fig.~\ref{fig:openfwi_clean} shows representative reconstructions
across the four families, while
Fig.~\ref{fig:openfwi_clean_bars} provides a quantitative comparison
of the metrics.
Unregularized FWI produces large errors with non-geological discontinuous
artifacts that violate the layered structure of the true models.
TV and Tikhonov regularization partially alleviate these instabilities,
but Tikhonov over-smooths the resulting reconstructions while TV preserves
sharp boundaries and performs quantitatively well on datasets whose
gradients are naturally sparse, such as FV-B, yet introduces visible
artifacts elsewhere, and neither of these classical regularizers are capable
of recovering the fine-scale detail of complex velocity models.

For diffusion-based methods, DiffusionFWI embeds the prior effectively on simple layered
structures, but lacks the capacity to capture complex faults and
fine-scale details; on the curved-boundary families it exhibits
visible artifacts, and its inversion quality is sensitive to the
smoothing hyperparameters in the sampling trajectory.
RED-DiffEq produces inversion results that are closer to the
ground truth with fewer artifacts, particularly on the CV-B and
CF-B families.
DLO further exhibits a more stable prior-preserving capability,
similarly captures discontinuous features such as faults, and more
accurately characterizes deep structures far from the surface; it
delivers the best metrics on CV-B, CF-B, and FF-B, and on FV-B it
ranks second
only to DiffusionFWI.

\begin{figure}[!htbp]
  \centering
  \includegraphics[width=1\linewidth]{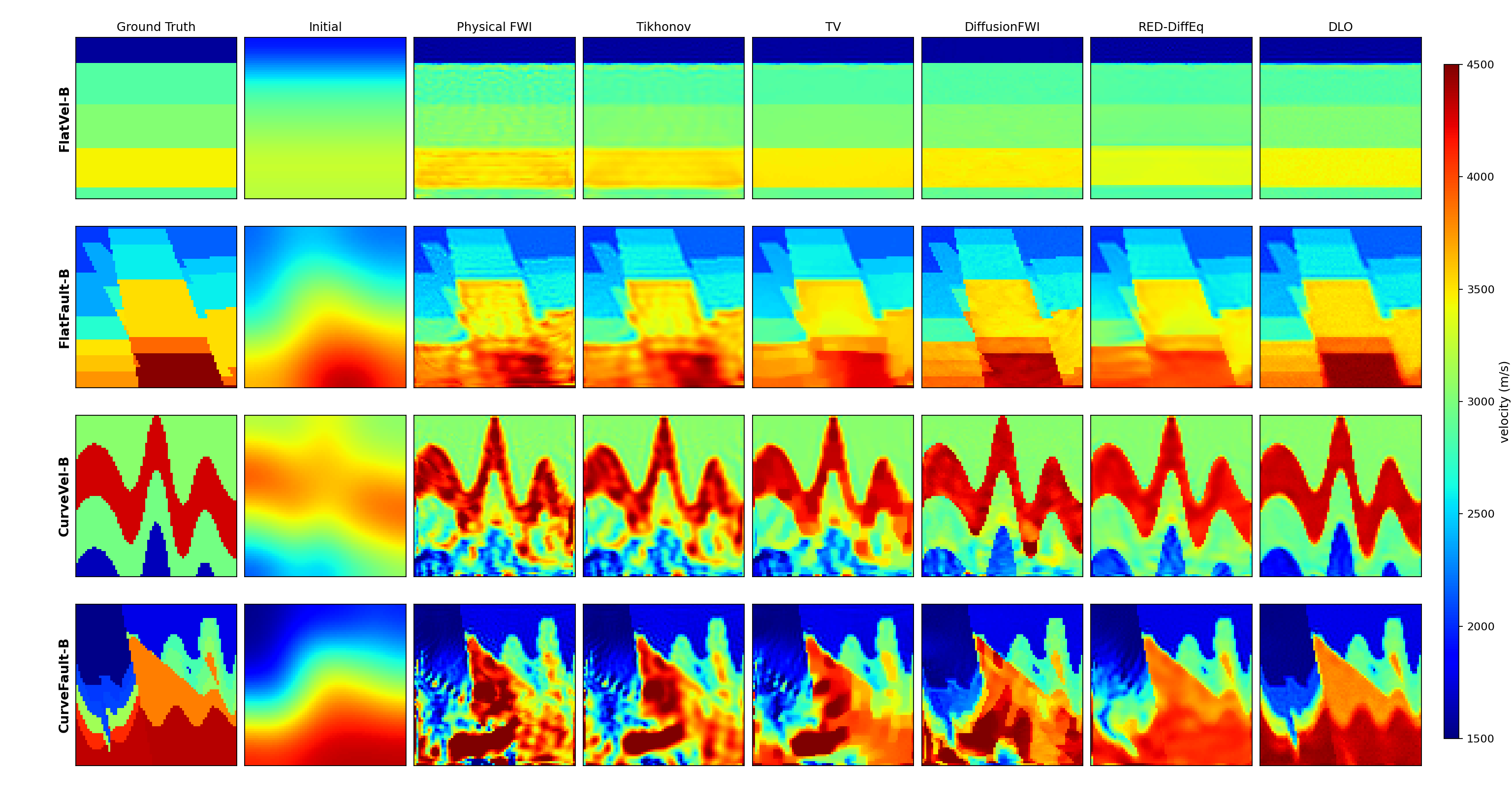}
  \caption{\textbf{OpenFWI clean-data qualitative comparison.}
    Representative reconstructions across CF-B, FV-B, FF-B, and CV-B for the
  ground truth, the smoothed initial model, and every baseline alongside DLO.}
  \label{fig:openfwi_clean}
\end{figure}

\begin{figure}[!htbp]
  \centering
  \includegraphics[width=\linewidth]{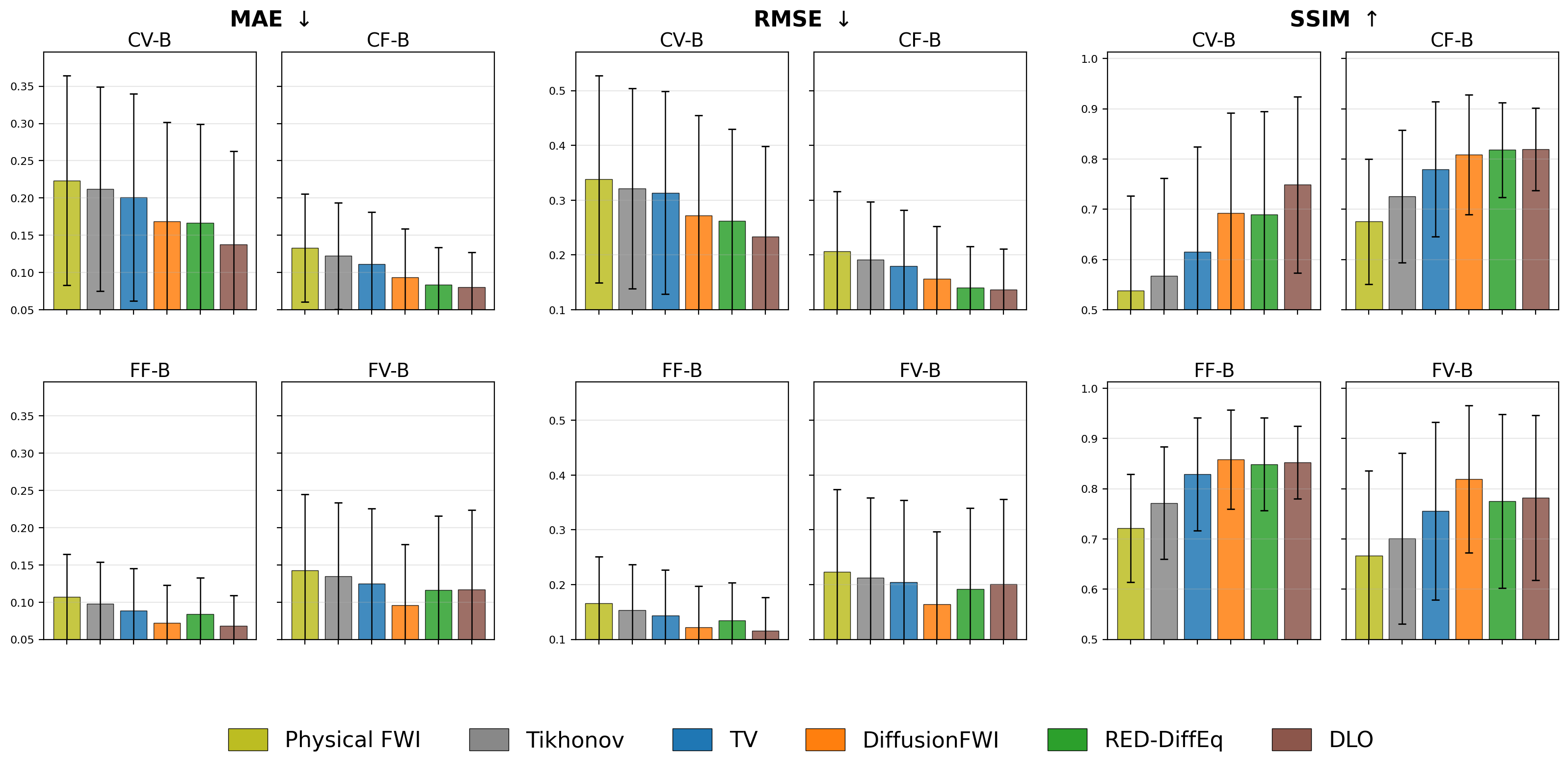}
  \caption{\textbf{OpenFWI clean-data quantitative comparison.}
    Normalized MAE, RMSE, and SSIM averaged over $100$ test samples per
    family. Within each metric block, the $2{\times}2$ grid arranges CV-B
    (top-left), CF-B (top-right), FF-B (bottom-left), and FV-B (bottom-right).
    Error bars indicate one standard deviation. Arrows next to each metric
    name indicate the preferred direction ($\downarrow$ lower is better,
  $\uparrow$ higher is better).}
  \label{fig:openfwi_clean_bars}
\end{figure}

\FloatBarrier
\subsubsection{Noisy Seismic Data}
\label{sec:exp_openfwi_noise}

We next assess robustness to additive measurement noise.
Following~\cite{shan2026regularization}, we perturb the seismic data with
independent Gaussian noise at three standard deviations,
$\sigma\!\in\!\{0.1,\,0.3,\,0.5\}$, corresponding to average per-sample
signal-to-noise ratios of $24.2$~dB, $14.6$~dB, and $10.2$~dB across the $400$
test samples.
All methods share the same hyperparameter settings and initialization
as in the clean-data scenario.

The comparison results of the Gaussian-noise corrupted scenario are
summarized in Fig.~\ref{fig:openfwi_noise}.
Unregularized FWI and Tikhonov-regularized FWI degrade sharply with
increasing $\sigma$, developing pronounced non-physical artifacts in regions
of low data sensitivity.
TV regularization substantially suppresses these artifacts but introduces a
clearly visible staircase pattern, particularly on the smooth-gradient
families, so the trade-off between artifact suppression and structural
fidelity is unfavorable.

Among the diffusion-based methods, RED-DiffEq and DLO are markedly robust
to noise: the learned prior suppresses noise-induced fluctuations that lie
outside the geological manifold, and their metrics degrade only mildly even
at $\sigma\!=\!0.5$.
DiffusionFWI, in contrast, performs competitively under clean data and low
noise, but as the noise level increases its reconstructions develop
large-area artifacts and non-prior features that deviate from the
geological structures learned by the diffusion model.
DLO consistently achieves the best inversion metrics across the vast
majority of experimental scenarios and noise levels.
Under noise-corrupted conditions, DLO maintains its data-prior-driven
inversion paradigm, demonstrating faithful recovery of fine-scale
velocity details and accurate reconstruction of layered and complex
boundary structures.

\begin{figure}[!htbp]
  \centering
  \begin{minipage}{\linewidth}\centering
    (\textbf{a}) \\[0.2em]
    \includegraphics[width=\linewidth]{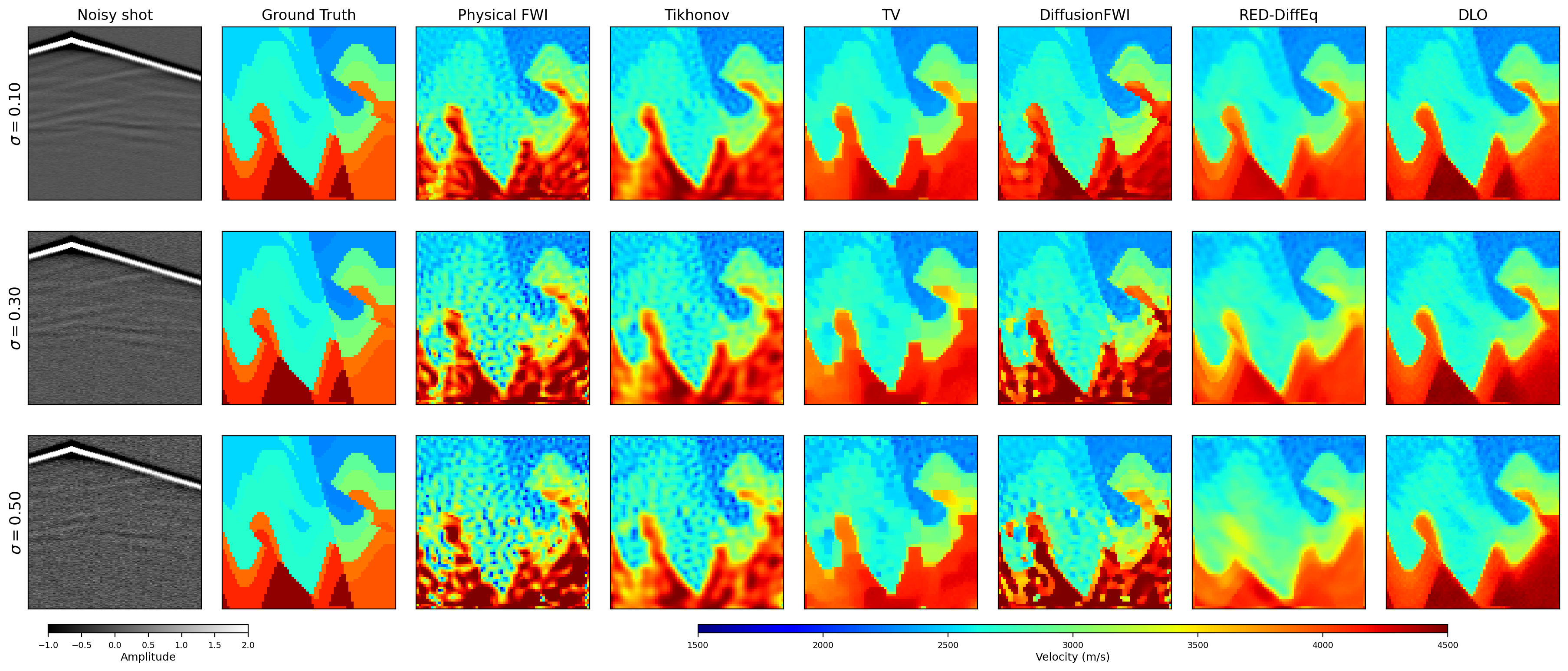}
  \end{minipage}\\[0.6em]
  \begin{minipage}{\linewidth}\centering
    (\textbf{b}) \\[0.2em]
    \includegraphics[width=\linewidth]{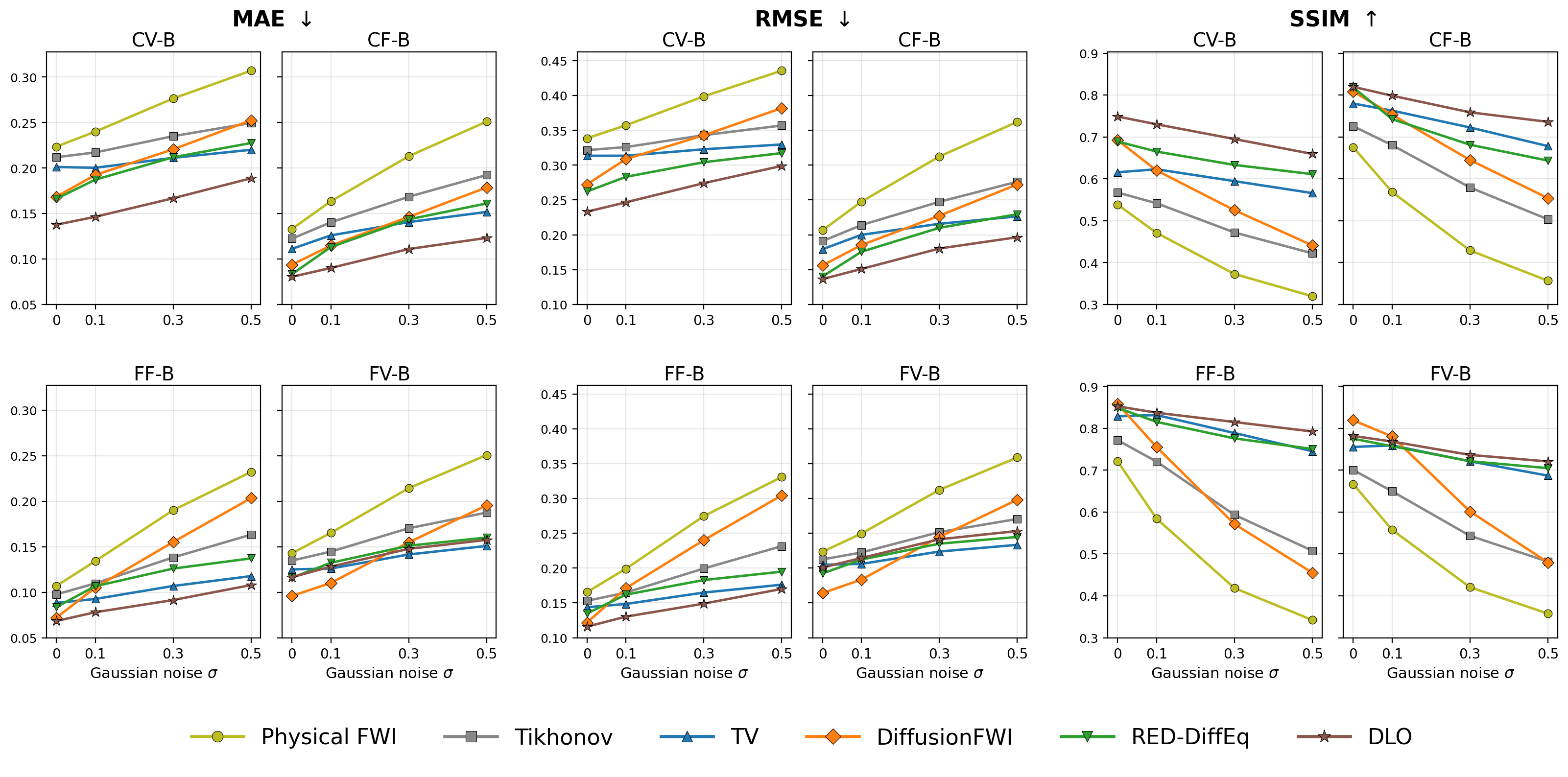}
  \end{minipage}
  \caption{\textbf{OpenFWI noisy-data results (Gaussian noise).}
    (\textbf{a}) Representative single-sample reconstructions on CF-B at three
    noise levels ($\sigma\!=\!0.10,\,0.30,\,0.50$); the leftmost column shows
    the noisy shot gather, the second column the ground-truth velocity, and the
    remaining columns the reconstructions from each method.
    (\textbf{b}) Normalized MAE, RMSE, and SSIM as a function of the Gaussian
  noise standard deviation, averaged over $100$ test samples per family.}
  \label{fig:openfwi_noise}
\end{figure}

\FloatBarrier
\subsubsection{Seismic Data with Missing Traces}
\label{sec:exp_openfwi_missing}

We finally evaluate the methods under incomplete acquisitions in which a
fixed subset of receivers is removed from every shot gather.
We consider three settings corresponding to \emph{slight}, \emph{half}, and
\emph{most-missing} acquisitions: $10$, $35$, and $60$ of the $70$ traces
are removed, which discards $14.3\%$, $50.0\%$, and $85.7\%$ of the
available recordings, respectively; the missing receiver indices are kept
identical across shots to mimic a realistic sensor-failure scenario.

Performance under missing traces is reported in
Fig.~\ref{fig:openfwi_missing}.
Almost every method degrades monotonically as the fraction of missing traces
grows, but the diffusion-based methods are noticeably more robust than the
non-prior baselines.
RED-DiffEq and DLO retain a coherent prior-consistent structure in the
slight- and half-missing regimes, and even in the most-missing setting they
avoid the conspicuous artifacts seen in the other methods while preserving
the dominant geological features.
Standard FWI and the classical-regularization baselines, in contrast, exhibit
pronounced artifacts and a global loss of fine-scale detail across the
velocity field as soon as the acquisition is meaningfully sparsified.
Taken together with the noisy-data experiments of
Section~\ref{sec:exp_openfwi_noise}, these results demonstrate that DLO
remains effective on the more realistic acquisitions encountered in
practice, including both heavily sparsified and noise-contaminated
recordings.

\begin{figure}[!htbp]
  \centering
  \begin{minipage}{\linewidth}\centering
    (\textbf{a}) \\[0.2em]
    \includegraphics[width=\linewidth]{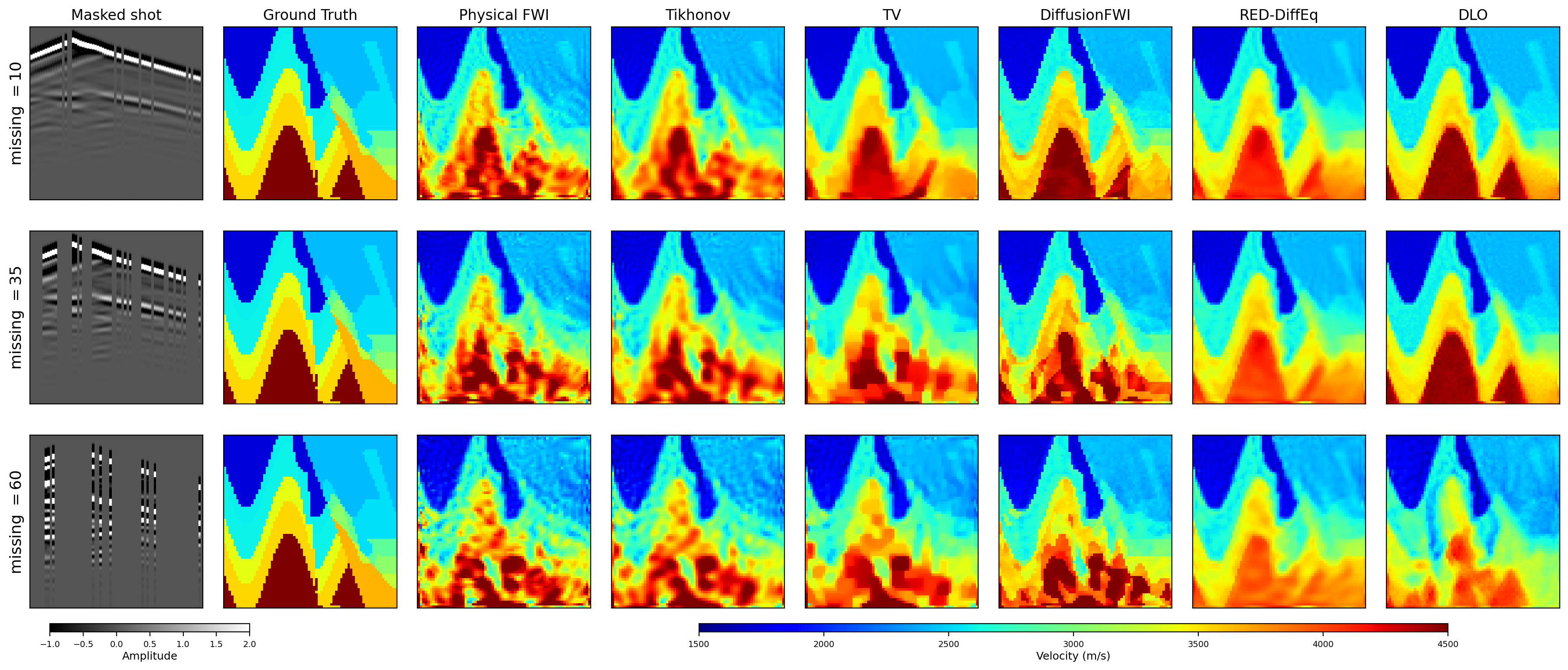}
  \end{minipage}\\[0.6em]
  \begin{minipage}{\linewidth}\centering
    (\textbf{b}) \\[0.2em]
    \includegraphics[width=\linewidth]{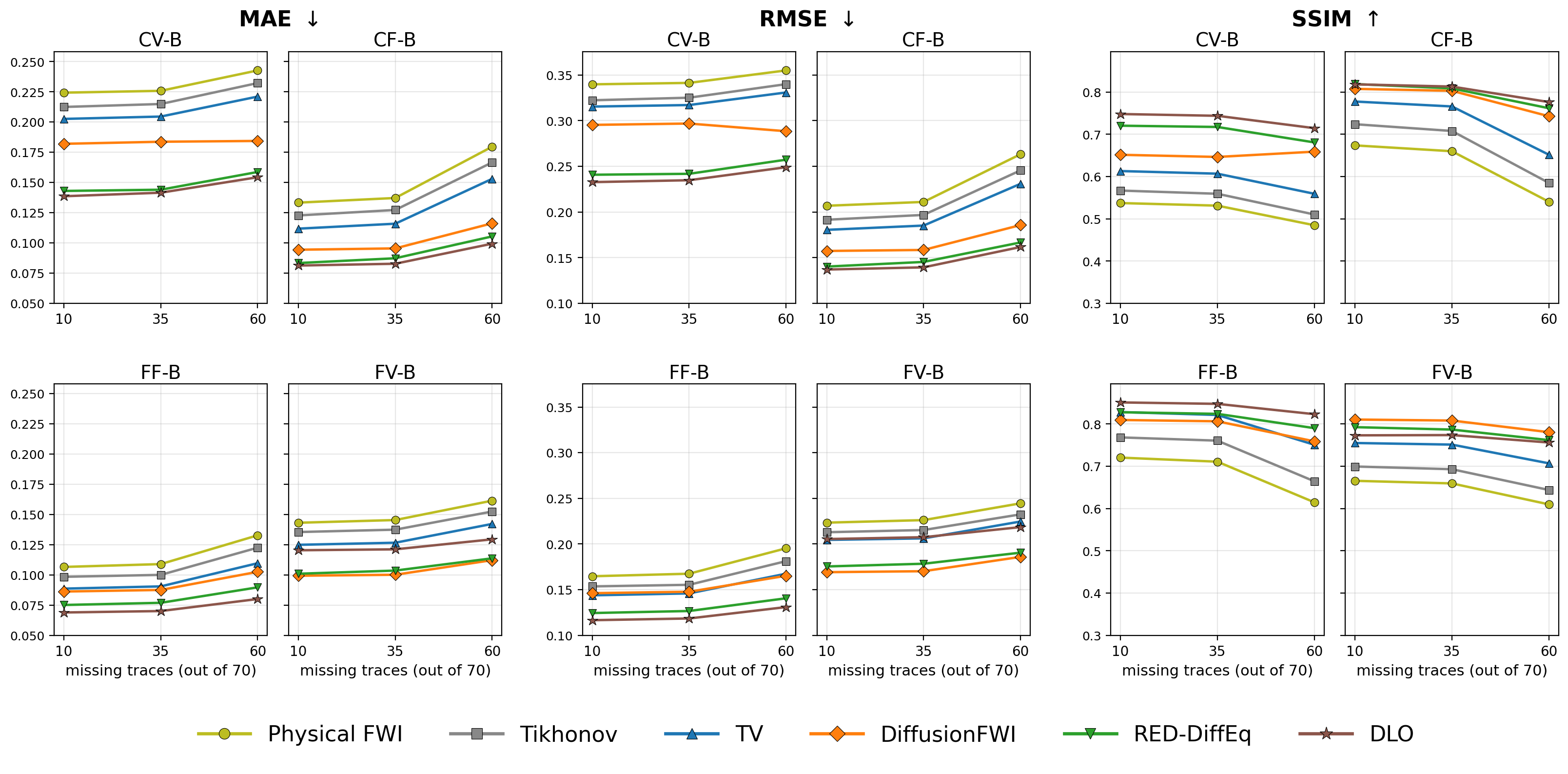}
  \end{minipage}
  \caption{\textbf{OpenFWI missing-trace results.}
    (\textbf{a}) Representative single-sample reconstructions on CF-B with
    $10$, $35$, and $60$ of the $70$ receivers removed; the leftmost column
    shows the masked shot gather, the second column the ground-truth velocity,
    and the remaining columns the reconstructions from each method.
    (\textbf{b}) Normalized MAE, RMSE, and SSIM as a function of the number of
  missing traces, averaged over $100$ test samples per family.}
  \label{fig:openfwi_missing}
\end{figure}

\subsection{Sensitivity to Latent Initialization}
\label{sec:exp_uq}

DLO is a deterministic method whose optimization result depends only on the
choice of hyperparameters and the initialization of the physical velocity
$v$ and the latent variable $z$.
To evaluate the sensitivity of DLO to the random initialization of the latent
variable, we conduct an experiment on the OpenFWI test set.
For each sample, we generate $20$ random initializations of
$z^{(0)} \sim \mathcal{N}(0,I)$; we take the ensemble mean over
the $20$ runs as the final prediction and report the standard deviation across
runs as a measure of initialization-induced uncertainty.

\begin{figure}[!htbp]
  \centering
  \includegraphics[width=0.98\linewidth]{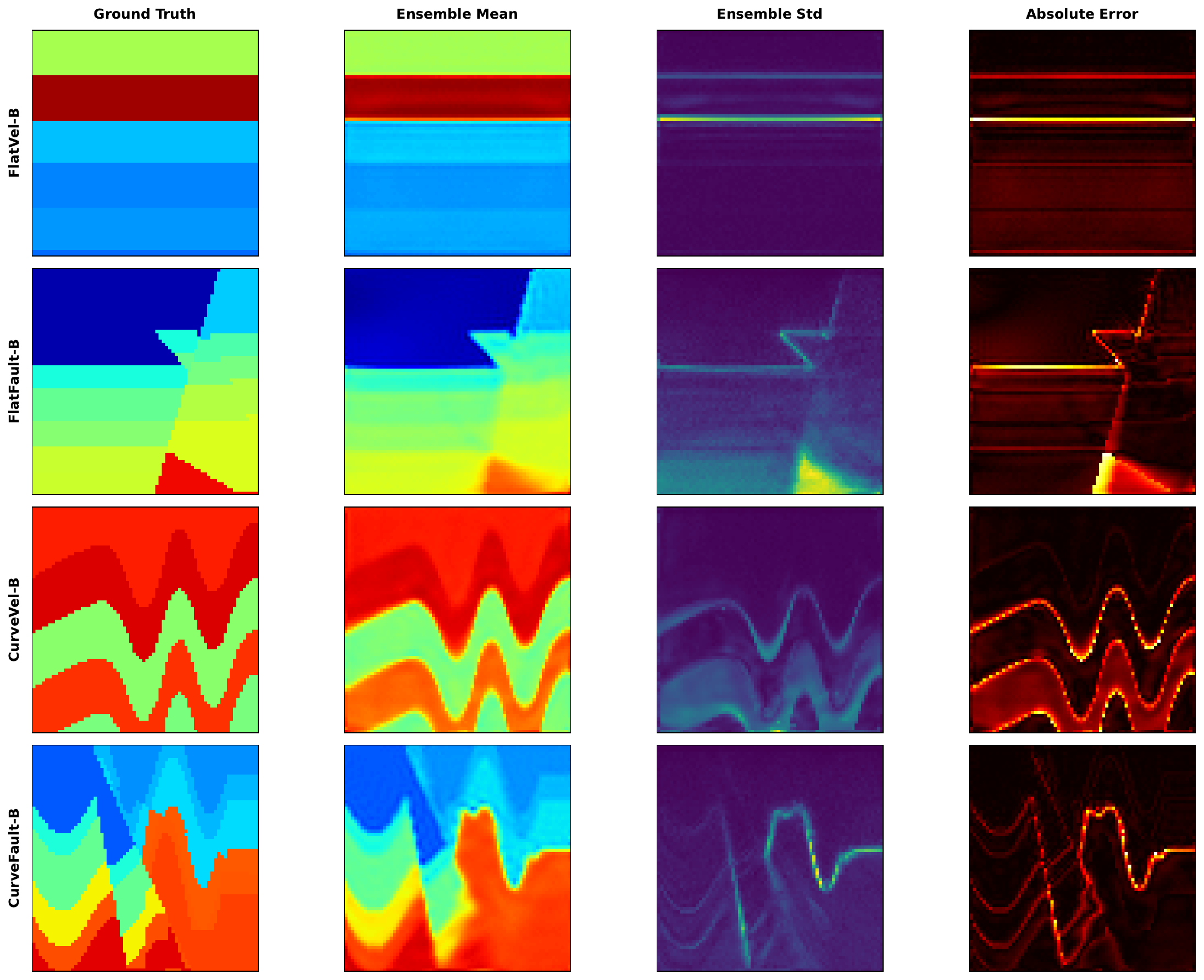}
  \caption{\textbf{Sensitivity of DLO to the latent initialization $z^{(0)}$.}
    Each row shows one representative sample per OpenFWI dataset.
    Columns (left to right): ground truth; ensemble mean over $20$ independent
    $z^{(0)}$ draws; per-pixel ensemble standard deviation;
    absolute error between ensemble mean and ground truth.
    Uncertainty is localized at geological boundaries, while the bulk velocity
  structure is robustly recovered across all initializations.}
  \label{fig:sensitivity}
\end{figure}

Fig.~\ref{fig:sensitivity} indicates that the per-pixel ensemble standard
deviation is primarily concentrated along fault interfaces and geological
discontinuities, as well as in regions where the absolute error is larger,
such as deeper layers with limited seismic illumination.
The spatial correspondence between the standard deviation and the absolute
error confirms that the initialization-induced uncertainty is positively
correlated with the inversion error.
To quantitatively assess the correspondence between initialization-induced
uncertainty and inversion error, we compute the per-pixel Spearman and Pearson
correlation coefficients between the ensemble standard deviation and the
absolute error of the ensemble mean, separately for each of the $40$ test
samples.
Across all samples, the mean Spearman rank correlation is $0.6454$ and the
mean Pearson correlation is $0.6282$, both significantly
positive ($p < 10^{-26}$, one-sample $t$-test against zero).
All $40$ samples yield positive correlations, confirming that the
initialization-induced uncertainty and the reconstruction error are
robustly co-oriented: pixels where the ensemble varies more across
independent runs are also pixels where the prediction tends to deviate more
from the ground truth.

\begin{figure}[!htbp]
  \centering
  \includegraphics[width=0.92\linewidth]{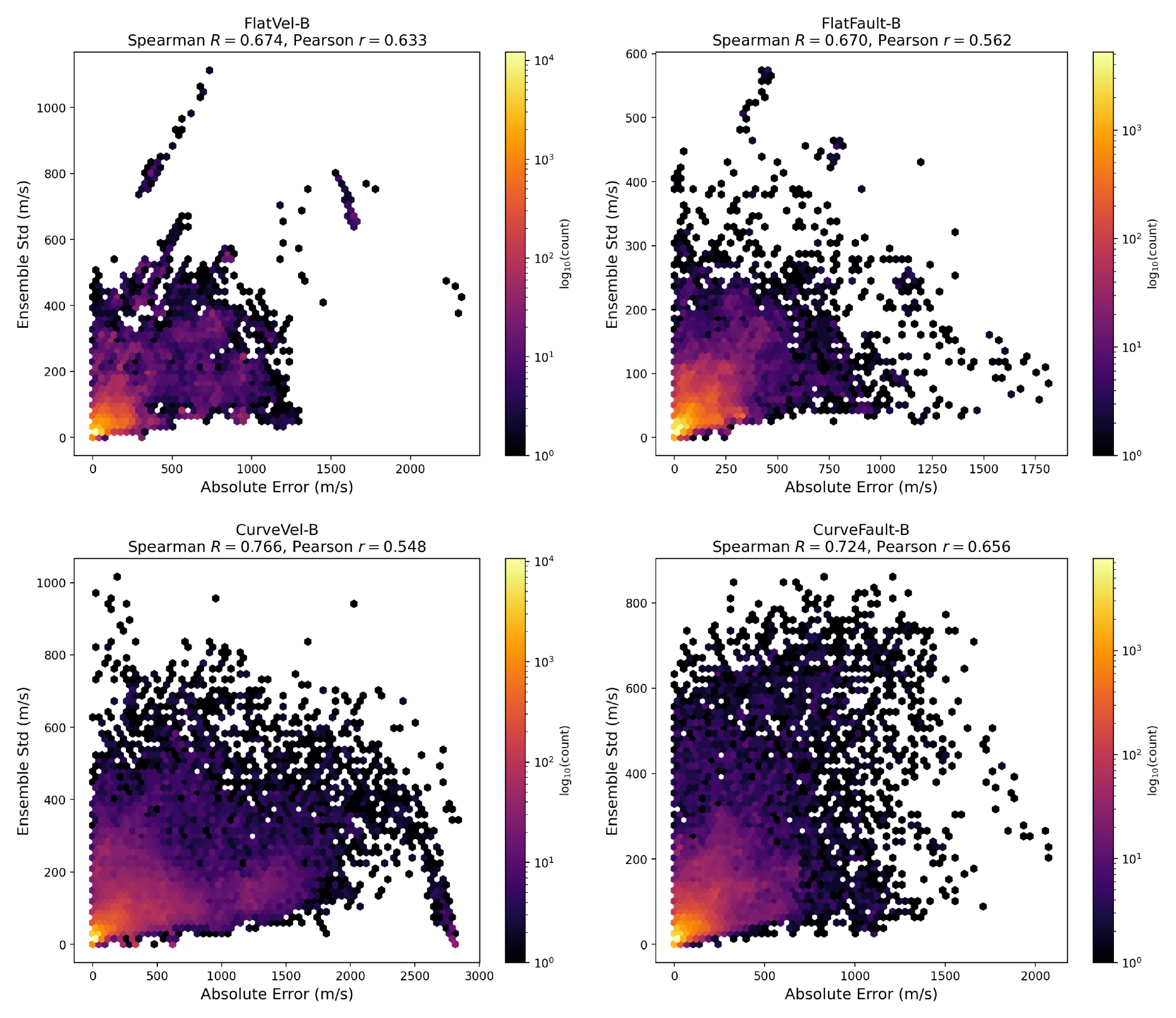}
  \caption{\textbf{Pixel-level ensemble standard deviation versus absolute
    error, aggregated over all test samples per dataset.}
    The high-density region exhibits a clear positive trend across all four
    geological families, further corroborating the per-sample correlation
  results.}
  \label{fig:pixel_std_vs_error}
\end{figure}

Fig.~\ref{fig:pixel_std_vs_error} shows the pixel-level joint distribution of
ensemble standard deviation and absolute error for each dataset.
The high-density region consistently follows a positive trend, independently
confirming that larger initialization-induced uncertainty reliably indicates
larger reconstruction error at the pixel level.
Across all $40$ test samples, the mean within-sample standard deviation of
RMSE, MAE, and SSIM are $0.0218$, $0.0146$, and $0.0285$, respectively.
These results confirm that DLO yields consistent reconstructions across
different initializations of $z^{(0)}$, while the positively correlated
ensemble uncertainty provides a reliable indicator of local reconstruction
quality that requires no additional computation beyond the ensemble itself.

\FloatBarrier
\subsection{Validation on the Marmousi and Overthrust Benchmarks}
\label{sec:exp_large}

To probe generalization beyond the OpenFWI training distribution, we extend
the experiments to the Marmousi~\cite{versteeg1994marmousi} and
Overthrust~\cite{aminzadeh19973} velocity models.
Marmousi is a 2D synthetic acoustic model built after the Cuanza basin
(Angola), with strongly dipping curved layers cut by a dense fault network;
the SEG/EAGE Overthrust model represents a thrust-belt setting with
overthrust sheets and high-velocity carbonate units beneath an erosional
surface.
Both targets differ markedly from the OpenFWI prior families, so these
experiments test whether our method extends to substantially more complex
geology and whether out-of-distribution inversion is feasible with a
diffusion prior reused as-is.

Representative reconstructions at $\sigma_\mathrm{init}\!=\!20$ on clean data
are shown in Figs.~\ref{fig:marmousi_panel} and~\ref{fig:overthrust_panel},
and the corresponding metrics are the clean column of
Figs.~\ref{fig:marmousi_sweep} and~\ref{fig:overthrust_sweep}.
DLO faithfully recovers the geological character of both targets, restoring
fault structures, dipping interfaces, and the discontinuous stratigraphy
that the classical regularizers blur away, and it attains the best MAE,
RMSE, and SSIM among all methods on both benchmarks.
RED-DiffEq substantially outperforms the classical baselines, confirming
that the patched diffusion regularizer transfers meaningfully out of
distribution, while DiffusionFWI lies between RED-DiffEq and the classical
methods.
We further test the dependence on the smoothed initial model by sweeping
$\sigma_\mathrm{init}\!\in\!\{24,28,32\}$ on noise-free data, and DLO
exhibits clear insensitivity to the initialization, retaining the best
metrics across every $\sigma_\mathrm{init}$ on both Marmousi and Overthrust.

We next assess robustness to additive Gaussian measurement noise by
perturbing the seismic data with $\sigma\!\in\!\{0.1,0.3,0.5\}$ at
$\sigma_\mathrm{init}\!=\!20$ without per-condition retuning
(Figs.~\ref{fig:marmousi_sweep} and~\ref{fig:overthrust_sweep}).
The classical baselines degrade markedly with increasing noise, developing
large-amplitude artifacts in regions of low data sensitivity.
DiffusionFWI lies between the classical baselines and the prior-regularized
methods on both benchmarks.
DLO and RED-DiffEq both display strong noise robustness, with metric curves
that stay nearly flat across the full noise range; DLO essentially
maintains the same MAE, RMSE, and SSIM level as RED-DiffEq as $\sigma$
grows, while retaining its advantage at the lower noise levels.
Taken together, these experiments demonstrate that DLO generalizes well
across diverse geological scenarios and preserves high-fidelity
reconstructions under the kinds of data corruption encountered in
practice, underscoring its potential for real-world seismic imaging
applications.

\begin{figure}[!htbp]
  \centering
  \includegraphics[width=0.98\linewidth]{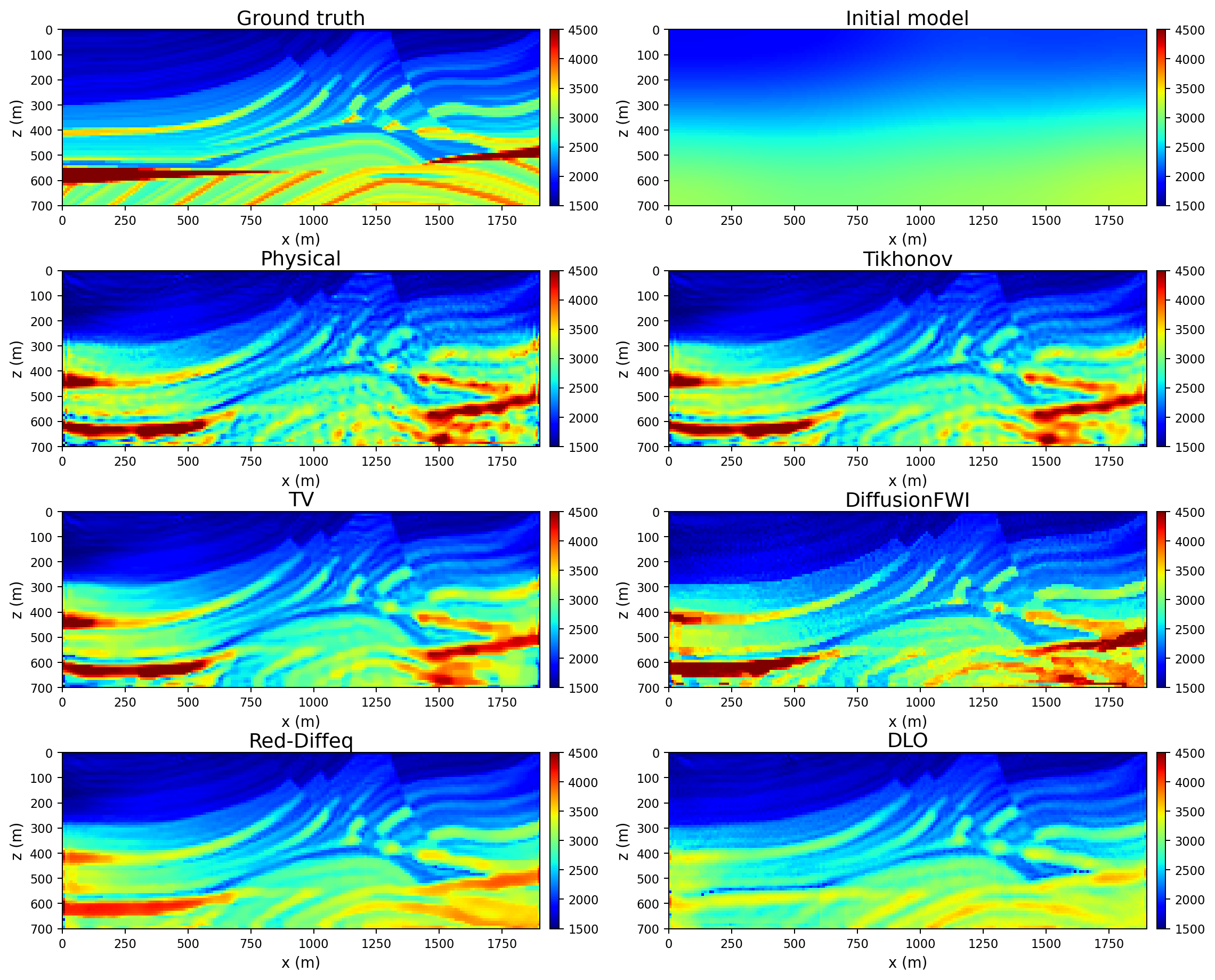}
  \caption{\textbf{Marmousi qualitative comparison
    ($\sigma_\mathrm{init}\!=\!20$, noise-free).}
    Ground truth, smoothed initial model, and reconstructions of every
    method on the $70\times190$ Marmousi benchmark. All diffusion-based
    methods reuse the CurveFault-B prior trained on OpenFWI; DLO and
    RED-DiffEq operate on three overlapping $70\times70$ patches and
    DiffusionFWI uses a sliding-window adaptation, while the classical
    baselines act on the full rectangular field.}
  \label{fig:marmousi_panel}
\end{figure}

\begin{figure}[!htbp]
  \centering
  \includegraphics[width=0.98\linewidth]{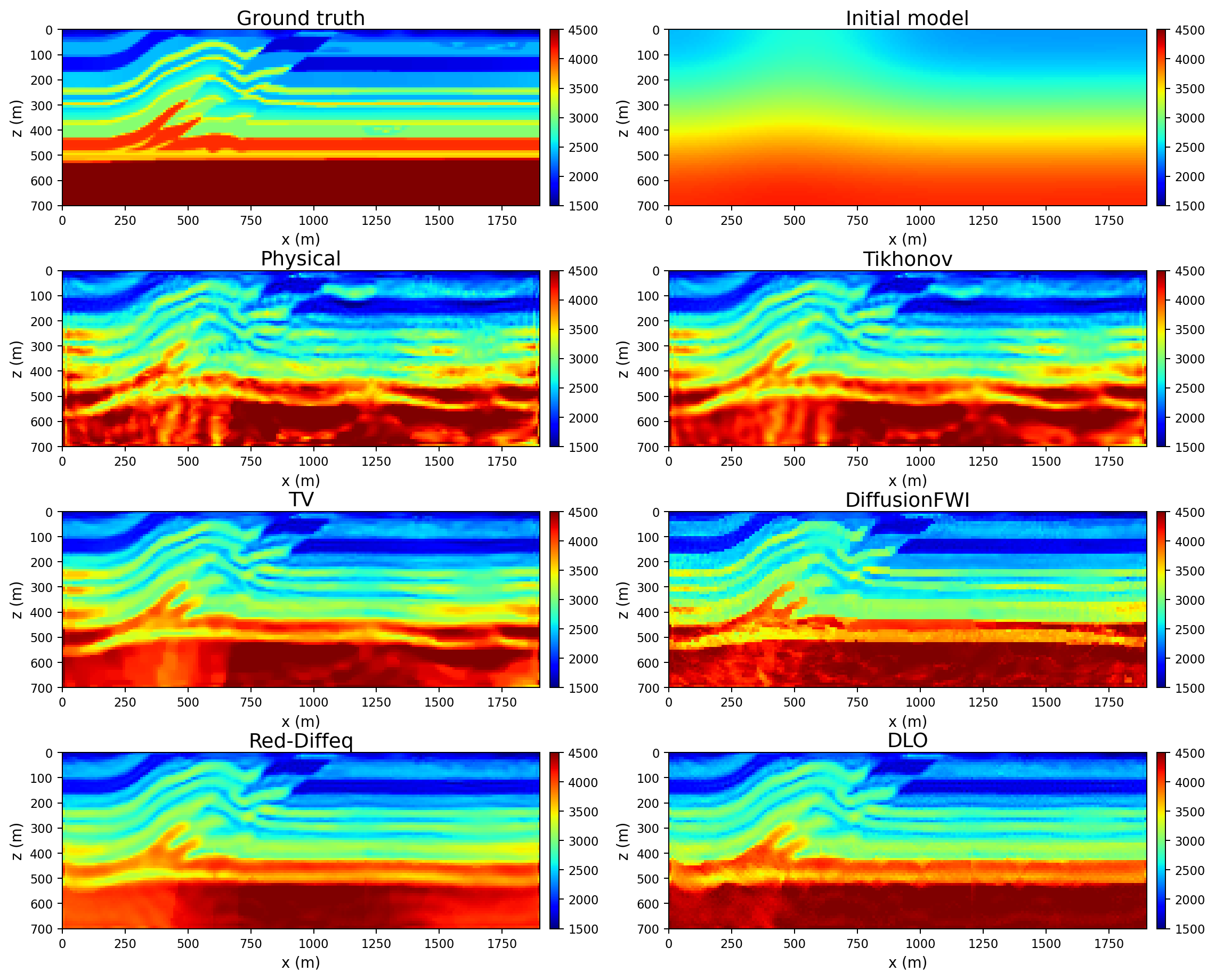}
  \caption{\textbf{Overthrust qualitative comparison
    ($\sigma_\mathrm{init}\!=\!20$, noise-free).}
    Ground truth, smoothed initial model, and reconstructions of every
    method on the $70\times190$ Overthrust benchmark, under the same
    setup as Fig.~\ref{fig:marmousi_panel}.}
  \label{fig:overthrust_panel}
\end{figure}

\begin{figure}[!htbp]
  \centering
  \includegraphics[width=0.98\linewidth]{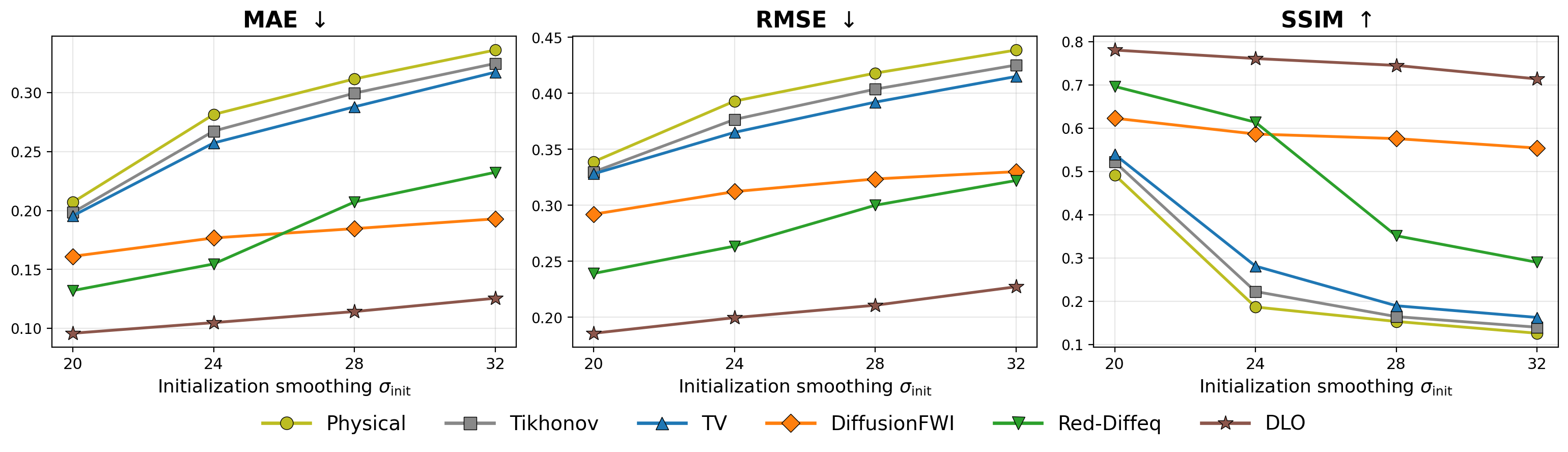}\\[0.4em]
  \includegraphics[width=0.98\linewidth]{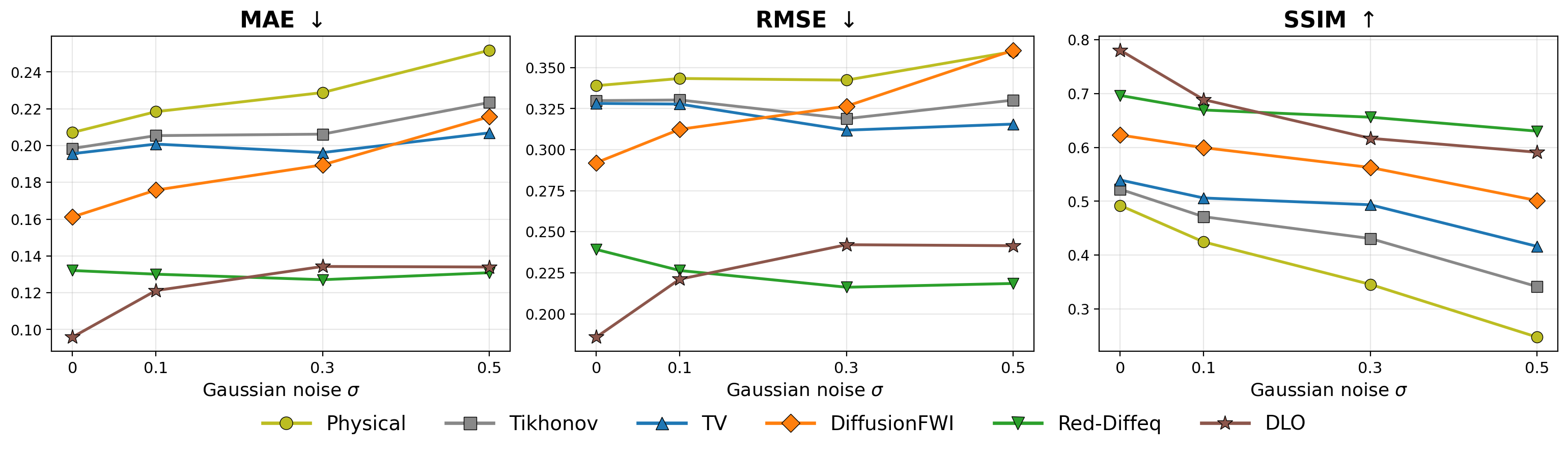}
  \caption{\textbf{Marmousi robustness sweep.}
    Normalized MAE, RMSE, and SSIM as a function of the
    initialization-smoothing kernel
    $\sigma_\mathrm{init}\!\in\!\{20,24,28,32\}$ on noise-free data
    (top) and as a function of the additive Gaussian-noise standard
    deviation $\sigma\!\in\!\{0,0.1,0.3,0.5\}$ at
    $\sigma_\mathrm{init}\!=\!20$ (bottom).}
  \label{fig:marmousi_sweep}
\end{figure}

\begin{figure}[!htbp]
  \centering
  \includegraphics[width=0.98\linewidth]{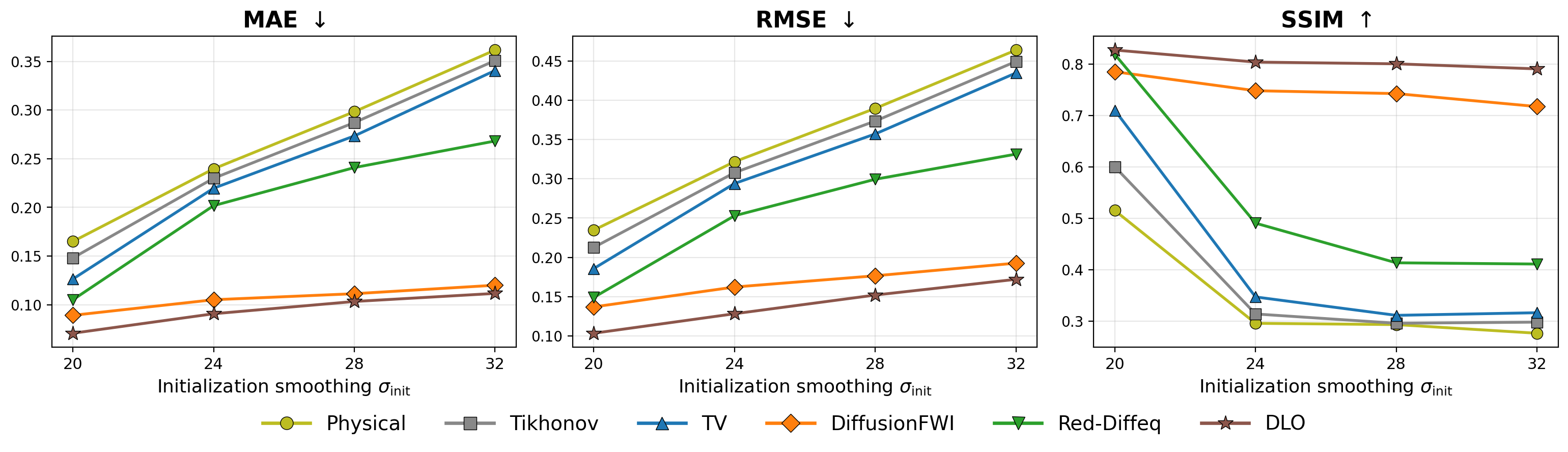}\\[0.4em]
  \includegraphics[width=0.98\linewidth]{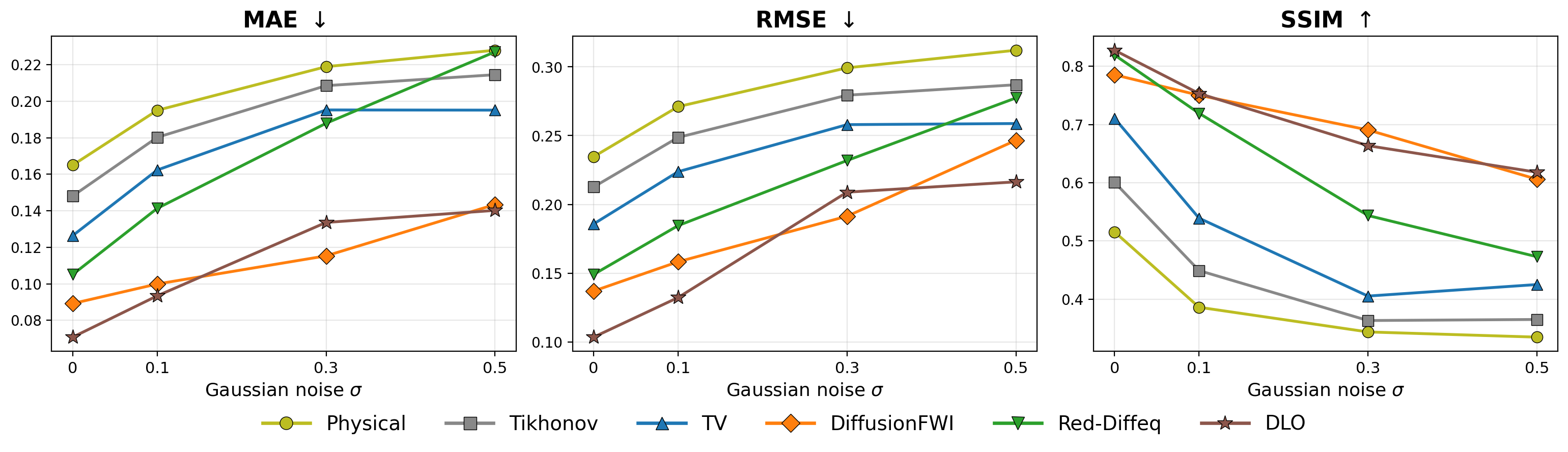}
  \caption{\textbf{Overthrust robustness sweep.}
    Same sweep as Fig.~\ref{fig:marmousi_sweep}, for the Overthrust
    benchmark.}
  \label{fig:overthrust_sweep}
\end{figure}

%% file: content/conclusion.tex
\section{Conclusion}
\label{sec:conclusion}

We proposed Decoupled Latent Optimization (DLO), a framework for embedding
pretrained diffusion priors into PDE-governed inverse problems.
DLO relaxes the standard latent-optimization formulation into a
quadratic-penalty objective that decouples the physical inversion from
the diffusion sampler. This decoupling isolates the data-fidelity
gradient in physical space, removes the need to backpropagate through
the composed PDE-solver--sampler chain, and permits conventional
physical-space initialization.

On the OpenFWI benchmark, DLO consistently outperforms classical
regularizers and existing diffusion-based methods under clean, noisy,
and missing-trace acquisitions. The prior, trained solely on the
$70\times70$ OpenFWI velocity models, transfers without modification to
the substantially larger Marmousi and Overthrust benchmarks, where DLO
recovers intricate faults and dipping layers and remains robust to
initialization smoothing and measurement noise.

The decoupled latent-optimization principle applies broadly to
PDE-governed inverse problems for which a pretrained generative prior
of the unknown field is available. Extending the framework to
three-dimensional geometries and to priors that capture richer
geological complexity are natural directions for future work.

%% file: content/appendix.tex
\appendix

\section{Diffusion Model Preliminaries}
\label{sec:bg_diffusion}

The essential definitions of the variance-preserving marginal, the denoiser
$D_\theta$, and the deterministic DDIM sampler $R_\theta$ are summarized in
Section~\ref{sec:diffusion_prelim} of the main text.
This appendix provides the complete mathematical derivations---the stochastic
differential equation, the probability-flow ODE, Tweedie's identity, and the
DDPM training objective---for reference.

A diffusion model defines a continuous family of Gaussian-corrupted marginals
$\{p_t\}_{t\in[0,T]}$ obtained by progressively adding noise to clean samples
$x_0\sim p_\mathrm{data}$.
Under the variance-preserving (VP) corruption~\cite{ho2020denoising,song2020score},
\begin{equation}
  \label{eq:vp_marginal_app}
  x_t = \sqrt{\bar\alpha_t}\,x_0 + \sqrt{1-\bar\alpha_t}\,\epsilon,
  \qquad \epsilon\sim\mathcal{N}(0,I),
\end{equation}
where $\bar\alpha_t\in(0,1]$ decreases monotonically from $\bar\alpha_0\!=\!1$ to
$\bar\alpha_T\!\approx\!0$, so $p_T\approx\mathcal{N}(0,I)$.
Conditional on $x_0$, $x_t$ is Gaussian, so the marginal density of
$x_t$ is the convolution of $p_\mathrm{data}$ with a rescaled Gaussian
kernel,
\begin{equation}
  \label{eq:vp_marginal_density}
  p_t(x)
  = \int p_\mathrm{data}(x_0)\,
  \mathcal{N}\!\bigl(x;\sqrt{\bar\alpha_t}\,x_0,\,(1-\bar\alpha_t)\,I\bigr)\,
  \mathrm{d}x_0.
\end{equation}
Viewing $\bar\alpha_t$ as a smooth function of $t$, the family $p_t(x)$
satisfies a drift--diffusion equation in $x$ and $t$,
\begin{equation}
  \label{eq:vp_fp}
  \partial_t p_t(x)
  = \tfrac{1}{2}\beta(t)\bigl[\nabla\!\cdot\!\bigl(x\,p_t(x)\bigr)
  + \nabla^2 p_t(x)\bigr],
  \qquad
  \beta(t) = -\frac{\dot{\bar\alpha}_t}{\bar\alpha_t}.
\end{equation}
The forward diffusion that induces these marginals is the It\^o SDE with drift
$-\tfrac{1}{2}\beta(t)\,x$ and diffusion coefficient $\sqrt{\beta(t)}$; sampling
reverses this process, and by Anderson's theorem~\cite{ANDERSON1982313} the
reverse-time dynamics that transport the standard Gaussian $p_T$ back to the
data distribution $p_\mathrm{data}$ are governed by the sampling SDE
\begin{equation}
  \label{eq:vp_sde}
  \mathrm{d}x
  = \Bigl[-\tfrac{1}{2}\beta(t)\,x
  - \beta(t)\,\nabla_{\!x}\log p_t(x)\Bigr]\mathrm{d}t
  + \sqrt{\beta(t)}\,\mathrm{d}\bar w,
\end{equation}
integrated backwards from $t=T$ to $t=0$, with $\bar w$ a reverse-time Wiener
process. The Fokker--Planck equation of this sampling SDE yields exactly the
same form as Eq.~\eqref{eq:vp_fp}, so it traverses the
marginals~\eqref{eq:vp_marginal_density} in reverse. The same marginals are
reproduced by the equivalent deterministic probability-flow ordinary
differential equation~\cite{song2020score}
\begin{equation}
  \label{eq:vp_ode}
  \mathrm{d}x
  = \Bigl[-\tfrac{1}{2}\beta(t)\,x
  - \tfrac{1}{2}\beta(t)\,\nabla_{\!x}\log p_t(x)\Bigr]\mathrm{d}t,
\end{equation}
which provides a deterministic sampler sharing the same marginals.
Both the SDE and the ODE require the score $\nabla_{\!x}\log p_t(x)$
of every noisy marginal, which by Tweedie's identity~\cite{efron2011tweedie}
admits the conditional-expectation form
\begin{equation}
  \label{eq:score_tweedie}
  \nabla_{\!x_t}\log p_t(x_t)
  = \frac{\sqrt{\bar\alpha_t}\,\mathbb{E}[x_0\!\mid\!x_t] - x_t}{1-\bar\alpha_t}.
\end{equation}
The conditional expectation $\mathbb{E}[x_0\!\mid\!x_t]$ is approximated
by a neural denoiser $D_\theta(x_t,t)$ trained by the standard denoising
objective
\begin{equation}
  \label{eq:ddpm_loss}
  \min_{\theta}\;
  \mathbb{E}_{\,x_0\sim p_\mathrm{data},\;
    t\sim\mathcal{U}\!(0,T),\;
  \epsilon\sim\mathcal{N}(0,I)}
  \bigl[\,\omega(t)\,\|D_\theta(x_t,t) - x_0\|^2\,\bigr],
  \qquad
  x_t = \sqrt{\bar\alpha_t}\,x_0 + \sqrt{1-\bar\alpha_t}\,\epsilon,
\end{equation}
which, substituted into Eq.~\eqref{eq:score_tweedie}, supplies a learned
score for the SDE~\eqref{eq:vp_sde} and the ODE~\eqref{eq:vp_ode}.
We adopt the deterministic DDIM update~\cite{song2020denoising}, a discretization of
the probability-flow ODE~\eqref{eq:vp_ode},
\begin{equation}
  \label{eq:ddim_update}
  x_{t-1} = \sqrt{\bar\alpha_{t-1}}\,D_\theta(x_t,t)
  + \sqrt{1-\bar\alpha_{t-1}}\,
  \frac{x_t-\sqrt{\bar\alpha_t}\,D_\theta(x_t,t)}{\sqrt{1-\bar\alpha_t}},
\end{equation}
which, iterated from $x_T=z\sim\mathcal{N}(0,I)$ down to $x_0$, defines a
deterministic sampler $R_\theta\!:\!z\!\mapsto\!x_0$ that transports the
standard Gaussian to (an approximation of) $p_\mathrm{data}$.
We treat $R_\theta$ as the learned generator throughout the paper.

\section{Implementation Details}
\label{sec:methods_bg}

\subsection{Diffusion Model Pretraining}
\label{sec:methods_diffusion}

For each of the four OpenFWI families (FV-B, FF-B, CV-B, CF-B) we
independently pretrain a velocity-domain diffusion model on its
training split.
All four families share the same network architecture, noise schedule,
and training protocol; the only difference between runs is the training
data.
Velocity fields are normalized by
$y = (v_{\,\mathrm{m/s}} - 3000)/1500$, so the physical range
$[1500,4500]$\,m/s maps to $[-1,1]$.
The native velocity resolution is $70\!\times\!70$; inside the network
we apply a small amount of reflection padding to $72\!\times\!72$ at
the input and slice the output back to $70\!\times\!70$, so the
denoiser operates on a resolution that is cleanly divisible at every
downsampling stage while the externally visible velocity field
remains $70\!\times\!70$ throughout pretraining and inversion.

\paragraph{Network architecture.}
The denoiser $\epsilon_\theta$ is a $2$D U-Net (the Hugging Face
Diffusers \texttt{UNet2DModel}) with one input and one output channel.
It comprises four resolution stages with output channel widths
$\{128,256,256,512\}$ and three downsampling steps that take the
padded input from $72\!\times\!72$ through $36\!\times\!36$ and
$18\!\times\!18$ down to a $9\!\times\!9$ bottleneck, mirrored by three
upsampling steps in the decoder.
A self-attention block is inserted at the $18\!\times\!18$ stage in
both the encoder and the decoder, with the remaining stages purely
convolutional, and timestep conditioning is supplied to every residual
block via a sinusoidal embedding.

\paragraph{Noise schedule.}
We use a discrete-time DDPM scheduler with $T\!=\!1000$ training
timesteps and a linear $\beta$ schedule from
$\beta_{1}\!=\!10^{-4}$ to $\beta_{T}\!=\!2\!\times\!10^{-2}$, yielding
the VP marginals defined in Section~\ref{sec:diffusion_prelim} (see also Appendix~\ref{sec:bg_diffusion} for the complete derivation).
The full procedure for constructing the discrete schedule
$\{\beta_t,\alpha_t,\bar\alpha_t\}_{t=1}^{T}$ from the two endpoint
values is summarized in Algorithm~\ref{alg:linear_beta}.
The network is trained to predict the injected noise $\epsilon$ by the
standard DDPM noise-prediction objective.
The same schedule is reused at inference time by the deterministic DDIM
sampler $R_\theta$ (Section~\ref{sec:app_algo}), with the number of
deterministic steps reduced from $T\!=\!1000$ to $n\!=\!3$ inside the
DLO loop.

\begin{algorithm}[!htbp]
  \caption{Linear $\beta$ schedule (DDPM, VP).}
  \label{alg:linear_beta}
  \begin{algorithmic}[1]
    \REQUIRE Total steps $T=1000$, start $\beta_1=10^{-4}$,
    end $\beta_T=2\!\times\!10^{-2}$
    \FOR{$t = 1$ {\bf to} $T$}
    \STATE $\beta_t \gets \beta_1 + \dfrac{t-1}{T-1}\,(\beta_T - \beta_1)$
    \STATE $\alpha_t \gets 1 - \beta_t$
    \STATE $\bar\alpha_t \gets \prod_{s=1}^{t}\alpha_s$
    \ENDFOR
    \RETURN $\{\beta_t,\alpha_t,\bar\alpha_t\}_{t=1}^{T}$
  \end{algorithmic}
\end{algorithm}

\paragraph{Training hyperparameters.}
We optimize the noise-prediction loss with Adam at peak learning rate
$10^{-4}$ following a cosine schedule with a $500$-step linear warm-up.
We train for $400$ epochs on each family at batch size $32$, which
gives roughly $1{,}500$ optimizer steps per epoch and a total of
$\sim\!6\!\times\!10^{5}$ update steps over the $48{,}000$-velocity
training set per family.
Training and validation loss curves for the CF-B run are shown in
Fig.~\ref{fig:ddpm_loss_curve}; the other three families exhibit
qualitatively identical behaviour.
After roughly $3\!\times\!10^{5}$ optimizer steps the loss begins to
oscillate around a slowly decreasing plateau, a common signature of
the stochastic noise-prediction objective in which the per-step
target $\epsilon$ is resampled and the loss therefore retains an
irreducible variance even once $\epsilon_\theta$ has effectively
converged.

\begin{figure}[!htbp]
  \centering
  \includegraphics[width=0.7\linewidth]{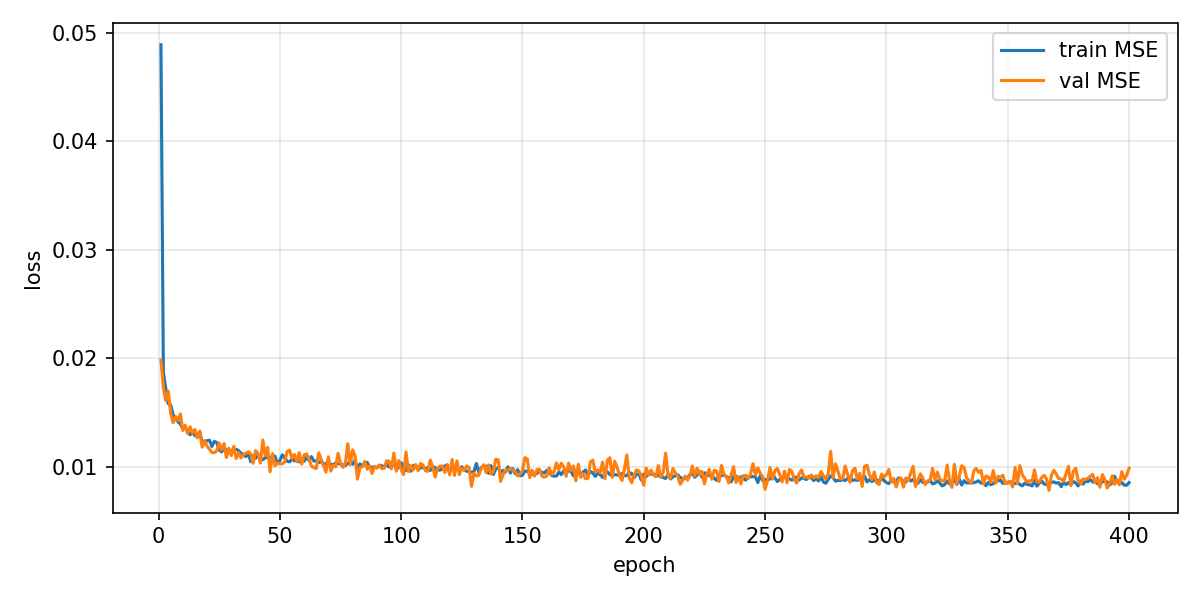}
  \caption{\textbf{Diffusion pretraining loss curve (CF-B).}
    Training and validation noise-prediction loss versus optimizer step
    for the CurveFault-B diffusion model. The remaining three families
  produce qualitatively identical curves.}
  \label{fig:ddpm_loss_curve}
\end{figure}

\subsection{Acoustic Wave-Equation Forward Solver}
\label{sec:app_solver}

The forward operator $\mathcal{F}$ in Eq.~\eqref{eq:fwi} numerically
integrates the 2D acoustic wave equation~\eqref{eq:wave} with a
Ricker wavelet source by a $4$th-order finite-difference scheme in
space and a leapfrog update in time, with a sponge absorbing layer
of quadratic damping along the four exterior boundaries.
Sources and receivers are uniformly distributed along the surface at
fixed depths.
The configuration follows the protocol used to generate the OpenFWI
datasets~\cite{deng2022openfwi}, and the same parameters are reused for the
Marmousi and Overthrust experiments, with the lateral extent
enlarged to accommodate the wider models.
Gradients of the data-misfit term with respect to the velocity model
are computed by the adjoint-state method~\cite{tarantola2005inverse}, which
back-propagates the data residual through a single reverse-time
solve of the same wave equation and contracts it with the stored
forward wavefield.
All parameters are summarized in Table~\ref{tab:forward_solver}.

\begin{table}[!htbp]
  \centering
  \caption{Forward-solver configuration. The same numerical scheme,
    source signature, and sponge boundary are used for both benchmark
    families; only the horizontal extent ($n_x$, $n_g$) changes between
  OpenFWI and the larger Marmousi/Overthrust models.}
  \label{tab:forward_solver}
  \begin{tabular}{lcccc}
    \toprule
    Parameter & Symbol & OpenFWI & Marmousi/Overthrust & Unit \\
    \midrule
    Horizontal grid points & $n_x$        & 70  & 190 & --- \\
    Vertical grid points   & $n_z$        & 70  & 70  & --- \\
    Grid spacing           & $\Delta x$   & 10.0 & 10.0 & m \\
    Time steps             & $n_t$        & 1000 & 1000 & --- \\
    Time step size         & $\Delta t$   & 0.001 & 0.001 & s \\
    Recording time         & $T$          & 1.0 & 1.0 & s \\
    Source frequency (Ricker) & $f$       & 15.0 & 15.0 & Hz \\
    Number of sources      & $n_s$        & 5   & 5   & --- \\
    Number of receivers    & $n_g$        & 70  & 190 & --- \\
    Source depth           & $z_s$        & 10  & 10  & m \\
    Receiver depth         & $z_g$        & 10  & 10  & m \\
    Boundary condition     & ---          & \multicolumn{2}{c}{Sponge layer (quadratic damping)} & --- \\
    ABC layer thickness    & $n_\mathrm{bc}$ & 120 & 120 & grid points \\
    \midrule
    Physical domain size   & ---          & $700\!\times\!700$  & $700\!\times\!1900$ & m$^{2}$ \\
    Seismic data shape     & ---          & $(5,1000,70)$ & $(5,1000,190)$ & --- \\
    \bottomrule
  \end{tabular}
\end{table}

\subsection{Hyperparameter Settings}
\label{sec:app_hyper}

To ensure a fair comparison, all methods share the same
hyperparameters across all scenarios.
For the data-misfit term, all methods minimize the $\ell_1$ norm
$\|\mathcal{F}(v)-\mathbf{d}_\mathrm{obs}\|_1$ rather than the squared
$\ell_2$ norm, as the $\ell_1$ loss provides greater robustness to
outliers in seismic recordings.
All optimization methods use the Adam optimizer with learning rate
$\eta=0.03$, run for $300$ outer iterations.
DLO uses the same optimizer for the velocity field $v$ as the other
methods, plus an independent Adam optimizer for the latent variable
$z$ with a constant learning rate $\eta_z=0.02$.
The regularization coefficients of Tikhonov ($\lambda\!=\!0.01$),
TV ($\lambda\!=\!0.01$), and RED-DiffEq ($\lambda\!=\!0.75$) follow
the values reported in RED-DiffEq~\cite{shan2026regularization} out of
empirical hyperparameter tuning, and the
manifold-tracking weight of DLO is fixed at $\lambda\!=\!0.5$
throughout.

DiffusionFWI~\cite{wang2023prior} follows a different optimization path,
in which FWI gradient steps are interleaved between successive
reverse-diffusion steps; we retain the protocol of the original code
release, starting the reverse sampler at $t\!=\!100$ on the
$T\!=\!1000$ schedule, inserting $10$ Adam steps (learning rate
$\eta\!=\!0.01$) between consecutive denoising steps.
For the stabilization techniques in the repo (gradient smoothing,
gradient normalization, and velocity-model Gaussian blur), our
cross-validation across the four OpenFWI families confirms that
velocity-model blur combined with gradient normalization most
reliably improves performance, and we enable only these two in all
reported DiffusionFWI results.

\subsection{Baseline Implementation Protocols}
\label{sec:app_baselines}

All baselines operate on velocity fields normalized to $[-1,1]$ by
$x = (v_{\,\mathrm{m/s}}-3000)/1500$ and are initialized from a
Gaussian-smoothed ground-truth model with standard deviation
$\sigma_\mathrm{init}=10$.

\paragraph{Tikhonov.}
We apply a first-order Tikhonov regularizer that penalizes the squared
gradient of the velocity model,
\begin{equation}
  \label{eq:tikhonov_reg}
  R_\mathrm{Tikhonov}(x)
  = \frac{1}{N}\sum_{i,j}
  \Bigl[(x_{i+1,j}-x_{i,j})^{2} + (x_{i,j+1}-x_{i,j})^{2}\Bigr],
\end{equation}
where $N$ is the number of grid points.
This penalty discourages abrupt variations and yields smooth solutions
at the cost of fine-scale detail.

\paragraph{Total variation.}
We use anisotropic TV, which promotes piecewise-constant structure and
preserves sharp velocity discontinuities,
\begin{equation}
  \label{eq:tv_reg}
  R_\mathrm{TV}(x)
  = \frac{1}{N}\sum_{i,j}
  \Bigl[\,|x_{i+1,j}-x_{i,j}| + |x_{i,j+1}-x_{i,j}|\,\Bigr],
\end{equation}
which has a known tendency to introduce staircase artifacts in regions
of smooth velocity gradient.

\paragraph{RED-DiffEq.}
We use the denoiser-based regularizer of
Eq.~\eqref{eq:prior_velocity}, with the time-dependent weighting in
the original formulation replaced by a fixed coefficient $\lambda$
following~\cite{shan2026regularization}.

\paragraph{DiffusionFWI.}
We adapt the official implementation of~\cite{wang2023prior} to our
acoustic forward solver, replacing the original elastic engine while
retaining the nested optimization protocol described in
Section~\ref{sec:app_hyper}.

\subsection{Computational Cost}
\label{sec:app_algo}

For stable optimization and computational efficiency, we instantiate the
DDIM sampler with $n=3$ deterministic steps as the prior velocity field
generator.
Table~\ref{tab:dlo_runtime} summarizes the per-iteration runtime breakdown
of DLO.
Compared to a standard FWI iteration (forward PDE solver plus adjoint
gradient), optimization with the additional DDIM-related operations costs
roughly $220\%$ of a standard FWI iteration under our implementation.

Compared to classical FWI methods, DLO loads a pretrained diffusion
model and therefore occupies additional GPU memory.
Peak GPU memory usage during the DLO step is approximately
$4600$\,MB, well below the capacity of modern GPUs.
The additional computational cost and memory usage at this level are
acceptable, given the inversion performance achieved by diffusion-based
methods such as DLO.

\begin{table}[!htbp]
  \centering
  \caption{\textbf{Per-iteration runtime breakdown of DLO.}
    Measured on the OpenFWI $70\!\times\!70$ velocity grid and averaged over
  $100$ iterations; $\pm$ denotes one standard deviation.}
  \label{tab:dlo_runtime}
  \begin{tabular}{lcc}
    \toprule
    Operation & Time (s) & Fraction \\
    \midrule
    Forward PDE solver               & $0.0294 \pm 0.0008$ & $15.1\%$ \\
    DDIM fwd$+$bwd ($z$-update)      & $0.0655 \pm 0.0036$ & $33.7\%$ \\
    DDIM decode (no grad)            & $0.0408 \pm 0.0032$ & $21.0\%$ \\
    $v$-backward (adjoint)             & $0.0587 \pm 0.0009$ & $30.2\%$ \\
    \midrule
    Total per DLO iteration          & $0.1946 \pm 0.0044$ & $100\%$ \\
    \bottomrule
  \end{tabular}
\end{table}